\documentclass{article} 
\usepackage{iclr2022_conference,times}
\iclrfinalcopy

\usepackage{amsmath,amsfonts,bm}

\def\eqref#1{equation~\ref{#1}}

\def\1{\bm{1}}

\DeclareMathAlphabet{\mathsfit}{\encodingdefault}{\sfdefault}{m}{sl}
\SetMathAlphabet{\mathsfit}{bold}{\encodingdefault}{\sfdefault}{bx}{n}

\usepackage{hyperref}
\usepackage{url}

\usepackage{amsmath}
\usepackage{amsfonts}
\usepackage{bbm}
\usepackage{tikz}
\usepackage{microtype}
\usepackage{booktabs}
\usepackage{wrapfig}
\usepackage{xcolor}
\usepackage{multirow}
\definecolor{DarkGreen}{rgb}{0.0,0.5,0.0}

\usepackage[english]{babel}
\usepackage{amsthm}
\newtheorem{assumption}{Assumption}

\newtheorem{theorem}{Theorem}
\newtheorem{lemma}{Lemma}
\newcommand{\sd}[1]{\scriptsize{#1}}

\usepackage{xspace}
\usepackage{mathtools}
\DeclarePairedDelimiterX{\infdivx}[2]{[}{]}{%
  #1\;\delimsize|\delimsize|\;#2%
}
\newcommand{\kld}[2]{\ensuremath{D_{KL}\infdivx{#1}{#2}}\xspace}
\newcommand{\abs}[1]{\left \lvert #1 \right \rvert}

\usepackage{float}
\usepackage{subfig}

\title{On a Benefit of Mask Language Modeling: \\ Robustness to Simplicity Bias}

\author{Ting-Rui Chiang \\
Carnegie Mellon University \\
\texttt{tingruic@andrew.cmu.edu}
}

\begin{document}

\maketitle

\begin{abstract}
Despite the success of pretrained masked language models (MLM), why MLM pretraining is useful is still a qeustion not fully answered.
In this work we theoretically and empirically show that MLM pretraining makes models robust to lexicon-level spurious features, partly answer the question.
We theoretically show that, when we can model the distribution of a spurious feature $\Pi$ conditioned on the context, then (1) $\Pi$ is at least as informative as the spurious feature, and (2) learning from $\Pi$ is at least as simple as learning from the spurious feature.
Therefore, MLM pretraining rescues the model from the simplicity bias caused by the spurious feature.
We also explore the efficacy of MLM pretraing in causal settings.
Finally we close the gap between our theories and the real world practices by conducting experiments on the hate speech detection and the name entity recognition tasks.
\end{abstract}

\section{Introduction}
Large-scale pretrained masked language models (MLM) is to pretrain a model that can predict tokens based on the context.
It has been shown useful for natural language processing (NLP) \citep{devlin-etal-2019-bert,liu2019roberta}.
Especially, \citet{gururangan-etal-2020-dont} shows that continuing the MLM pretraining with unlabeled target data can further improve the performance on downstream tasks.
However, the question "\textit{why is masked language model pretraining useful?}" has not been totally answered.
In this work, as a initial step toward the answer, we show and explain that MLM pretraining makes the model robust to lexicon-level spurious features.

\begin{wrapfigure}{r}{4.5cm}
    \centering
    \includegraphics[width=4cm]{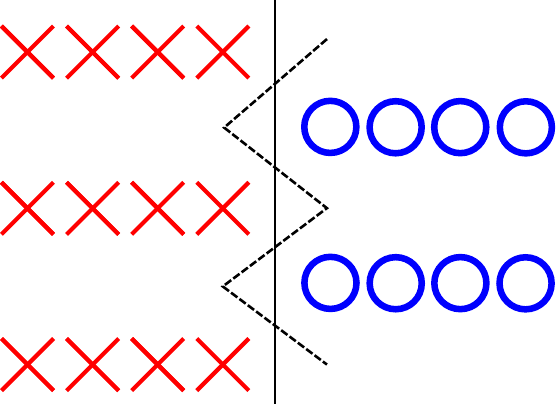}
    \caption{The pitfall of simplicity bias: The solid line is a simple (linear) decision boundary that utilizes only one dimension, while the dashed line is a more complex decision boundary that utilizes two dimensions and maximizes the margin.}
    \label{fig:simplicity-bias}
\end{wrapfigure}

Previous studies have empirically shown the robustness of MLM pretrained models.
\citet{hao-etal-2019-visualizing} show that MLM pretraining leads to wider optima and better generalization capability.
\citet{hendrycks-etal-2020-pretrained} and \citet{tu-etal-2020-empirical} show that pretrained models are more robust to out-of-distribution data and spurious features.
However, it remains unanswered why pretrained models are more robust.

We conjecture that models trained from scratch suffer from the pitfall of simplicity bias \cite{shah2020pitfalls} (Figure~\ref{fig:simplicity-bias}).
\citet{shah2020pitfalls} and \citet{kalimeris2019sgd} show that deep networks tend to converge to a simple decision boundary that involves only few features.
The networks may not utilize all the features and thus may not maximize the margin, which results in worse robustness.
A consequence of this could be that a model may excessively rely on a feature that has spurious association with the label and ignore the other features that are more robust.
While \citet{shah2020pitfalls} and \citet{kalimeris2019sgd} investigate networks with continuous input, in Section~\ref{sec:toy-example} we demonstrate how spurious features cause problems when the inputs are discrete.
Our experimental setting is more related to NLP, where the inputs are discrete.

We start the exploration with the following assumptions:
Let the sentence, label pair be $X, Y$.
\begin{assumption}
\label{assumption:decomposable}
We assume that from $X$, we can extract two features $X_1$ and $X_2$.
\end{assumption}
\begin{assumption}
\label{assumption:spurious}
$X_1$ is a spurious feature that has strong association with $Y$.
Specifically, it means that, solely relying on $X_1$, one can predict with high accuracy, but cannot be 100\% correctly.
\end{assumption}
\begin{assumption}
\label{assumption:deterministic}
$X_2$ is a robust feature based on which $Y$ can be predicted with 100\% accuracy. Namely, there exists a deterministic mapping $f_{X_2 \to Y}$ that maps $X_2$ to $Y$.
\end{assumption}

The assumptions above are realistic in some NLP tasks.
In NLP tasks, the input $X$ is a sequence of tokens.
Some tasks satisfy Assumption~\ref{assumption:decomposable}: $X$ can be decomposed into $X_1$ and $X_2$, where $X_1$ is the presence of certain tokens, and $X_2$ is the context of the token.
As a result, $X_2$ has much higher dimensionality than $X_1$.
Therefore, when Assumption~\ref{assumption:spurious} is true, due to the simplicity bias, a deep model is likely to excessively rely on $X_1$ and to rely on $X_2$ less.
However, if Assumption~\ref{assumption:deterministic} is true, we would desire the model to rely on $X_2$, which contains the semantic of the input $X$.  

With these assumptions, in Section~\ref{sec:informative} and Section~\ref{sec:easy} we provide a theoretical explanation how MLM pretraining makes a model robust to spurious features.
Let $\Pi^{(1)}$ be the conditional probability $P(X_1 | X_2)$.
We show (1) the relation between the mutual information $I(\Pi^{(1)}; Y) \geq I(X_1; Y)$ and that (2) the convergence rate of learning from $\Pi^{(1)}$ is of the same order as learning from $X_1$.
That is, when the MLM model can perfectly model the probability $P(X_1 | X_2)$ and thus generate perfect $\Pi^{(1)}$, learning from $\Pi^{(1)}$ is as easy as learning from $X_1$.
As a result, the model will be more likely to avoid the pitfall of simplicity of bias and to rely on $\Pi^{(1)}$.
Since $\Pi^{(1)}$ is estimated based on $X_2$, higher reliance on $\Pi^{(1)}$ also implies higher reliance on the robust feature $X_2$.
To relax Assumption~\ref{assumption:deterministic}, we make one step further by considering causal settings in Section~\ref{sec:causal}.

The above results partly explain why MLM pretrining is useful for NLP.
Denote a sequence of tokens as $X = \langle X_1, X_2, \cdots, X_L \rangle$.
During the MLM pretraining process, each token is masked randomly at certain probability, and the training objective is to predict the masked tokens with maximum likelihood loss.
As a result, the model is capable of estimating the conditional probability $P(X_i | X \setminus X_i )$ for all $i = 1, 2, \cdots, L$.
Even though which of the tokens is spurious is unknown, as long as the spurious token has non-zero probability to be masked during pretraining, MLM can estimated its distribution conditioned on the context and thus can reduce the reliance on it.

Finally, we close the gap between our theories and reality.
One major gap is that, in reality, we do not use the conditional probability for downstream tasks.
Instead, we feed the input $X$ without masking any token and fine-tune the model along with a shallow layer over its output.
Regardless of that, we hypothesize that the robustness to spurious tokens brought by MLM pretraining still exists.
To prove that, in Section~\ref{sec:toy-example-ptr} we use the toy example and verify the effect of MLM pretraining when using the common practice for fine-tuning.
In Section~\ref{sec:experiments} we validate our theories with two real world NLP tasks.

To sum up, our study leads to new research directions.
Firstly, we provide a new explanation of MLM pretraiing's efficacy.
Unlike the previous purely theoretical studies \cite{saunshi2021a,wei2021pretrained}, our assumptions are milder and more realistic.
Secondly, we study NLP robustness in a new perspective.
Many of the previous studies on robust NLP focus on supervised learning \cite{wang-etal-2021-dynamically,utama-etal-2020-towards,utama-etal-2020-mind,karimi-mahabadi-etal-2020-end,chang2020invariant,he-etal-2019-unlearn,Sagawa2020Distributionally,kennedy-etal-2020-contextualizing}.
However, without self-supervised learning, a model can impossibly extrapolate to out-of-distribution data when the domain shifts.
Our work complement previous studies that focus on the bias or robustness of a model generated by the pretraining process \cite{kumar-etal-2020-nurse,hawkins-etal-2020-investigating,vargas-cotterell-2020-exploring,liu-etal-2020-mitigating,gonen-goldberg-2019-lipstick,kurita-etal-2019-measuring,zhao-etal-2019-gender}. In this work we investigate the pretraining process itself.
Better understanding of pretraining should be important for future research.

\section{A Toy Example}
\label{sec:toy-example}

With the assumptions, the following toy example and experiment will show that spurious features can cause the difficulty of convergence.
Denote the one-hot vector whose $i$th element is $1$ as $e_i$.
Define $\nu < 0.5$ as a flip rate, and let the dimension of the random variables $X_1$ and $X_2$ be $2$ and $d_2$ respectively.
Their value $x_1 \in \{e_1, e_2\}$ and $x_2 \in \{e_1, e_2, \cdots, e_{d_2} \}$.
Let the middle part of $X_2$, the dimensions from $d_2 / 2 - \nu d_2$ to $d_2 / 2 + \nu d_2$, be $\dot{X}_2$.
We consider a joint distribution where $X_1 = Y$ with probability $1 - \nu$, and $X_1 = 1 - Y$ when $\dot{X}_2 \ne 0$.
Specifically, we consider the following random process:

\begin{equation}
    \begin{split}
        & X_2 = e_i, i \sim \mathrm{Uniform}(1, d_2) \\
        & Y = \begin{cases}
            -1 & \text{if $X_2 = e_i$ for some even $i$} \\
            +1 & \text{otherwise}
        \end{cases} \\
        & X_1 = \begin{cases}
            X_2 & \text{If $X_2 = e_i$ for some $i$ such that $\lvert i - d_2 / 2 \rvert > \nu d_2 / 2$} \\
            e_i, i \sim \mathrm{Uniform}(1, 2) & \text{otherwise}
        \end{cases}
    \end{split}
    \label{eq:toy-dist}
\end{equation}
Namely $X_1$ is flipped with probability 0.5 when the index of the none-zero element of $X_1$ is within $d_2 / 2 - \lceil d_2 \nu \rceil$ and $d_2 / 2 + \lceil d_2 \nu \rceil$.
In this case, predicting $Y$ solely based on the spurious feature $X_1$ can only achieve accuracy $1 - \nu$.


We conduct experiments to inspect the effect of the spurious feature $X_1$ in this toy model.
We train linear networks by drawing batches of i.i.d. $([X_1; X_2], Y)$ pairs from the random process defined in \ref{eq:toy-dist}.
We use Adam optimization with learning rate 0.001 and the cross-entropy loss.
In addition to single-layer linear networks, we also try over-parameterized 2-layer and 3-layer linear networks.
The hidden size is [10, 32].
Since it is a linear separable problem, we can check whether the learned weight can lead to 100\% accuracy in the defined distribution.
We check it every 25 iterations.
We say a model has converged if it is 100\% accurate for 5 consecutive checks.
We report the number of the iterations required before it converges for different $\nu$ and $d_2$.

Even though it is a linear-separable convex optimization problem, our results in Table~\ref{tab:toy-exp-ptr} show that the spurious feature can impact the number of iterations required to converge.
We observe that when $\nu < 0.5$, the models tend to be trapped by the spurious feature, sticking at accuracy $1 - \nu$ for iterations.
When the spurious relation between $X_1$ and $Y$ is stronger, i.e. $\nu$ is smaller, the number of iterations required to converge is larger.
In addition, the number of iterations is also larger when the $d_2$ is larger.
An intuitive explanation is that the learning signal from $X_2$ is more sparse when $d_2$ is larger.

\section{A Theoretical Explanation of the Efficacy of MLM Pretraining}

\subsection{$P(X_1| X_2)$ is More Informative Than $X_2$}
\label{sec:informative}

The toy example above motivate us to consider the information contained in $P(X_1 | X_2)$.
In the toy example, when predicting $P(Y = 0| X)$, if we simply output $P(X_1 = (1, 0)| X_2)$, then the accuracy of our prediction of $Y$ will be as high as predicting $Y$ solely based on $X_2$.
It motivates us to inspect the "reliability" of the estimated $P(X_1 | X_2)$ as a feature for the prediction of $Y$ compared to $X_1$.
Let $\Pi_1 = P(X_1 | X_2)$ be a $|\mathcal{X}_1|$-dimensional random variable \footnote{We will omit the subscript of $\Pi_1$ when there is no ambiguity.}.
In this section we prove that when $P(X_1 | X_2)$ is estimated perfectly, $\Pi_1$ is at least as informative as $X_1$.

\begin{lemma}
\label{lemma:mi-pi}
 If it $\Pi_1$ perfect, namely $\Pi_1 = P( \cdot | X_2)$, then the mutual information $I(X_1; \Pi_1) = I(X_1; X_2)$.
\end{lemma}
\begin{proof}
Since $X_2$ is discrete, $\Pi_1$ is discrete too.
\begin{align*}
    H(X_1 | \Pi_1)
    =& \sum_{x_1, \pi_1} P(X_1, \pi_1) \log P(x_1 | \pi_1) \\
    =& \sum_{x_1, \pi_1} \sum_{x_2: P(X_1 | x_2) = \pi_1} P(x_1, x_2) \log P(x_1 | x_2) = H(X_1 | X_2)
\end{align*}
\end{proof}
Compared to previous works \cite{hjelm2018learning,pmlr-v80-belghazi18a,oord2018representation} that show some training objectives similar to MLM's are lower bounds of the mutual information $I(X_1; X_2)$, we directly show that the output of the MLM, $\Pi$, maximizes the mutual information, since $I(X_1; f(X_2)) \leq I(X_1; X_2)$ for any $f$.
Moreover, instead of explaining the efficacy of pretraining with the infomax principle \cite{infomax1,bell1995information}, our theories below provide a different perspective.

\begin{theorem}
\label{thm:mi-pi}
If $\Pi$ is perfect,
\begin{equation}
    I(\Pi; Y) \geq I(X_1; Y)
\end{equation}
\end{theorem}
\begin{proof}
Since $\Pi$ is perfect, by Lemma~\ref{lemma:mi-pi}, we have
\begin{equation}
    I(X_1, X_2) =
    I(\Pi, X_2).
    \label{eq:optimal-estimator}
\end{equation}
By data processing inequality, the first equality in Equation~\ref{eq:optimal-estimator} implies $I(X_1, X_2 | \Pi) = 0$.
By Assumption~\ref{assumption:deterministic}, a deterministic mapping $f_{X_2 \to Y}$ from $X_2$ to $Y$ exists. Applying data processing inequality again, we have
\begin{equation}
\begin{split}
    I(X_1, X_2 | \Pi) \geq I(f_{X_1 \to Y}(X_1), X_2 | \Pi)
    = I(Y, X_2 | \Pi) \geq 0,
\end{split}
\label{eq:mi-pi-data-processing}
\end{equation}
which implies $I(Y, X_2 | \Pi) = 0$. Accordingly,
\begin{align}
    H(Y | \Pi) = H(Y | X_2, \Pi) \leq & H(Y | X_2)
\end{align}
\end{proof}

Theorem~\ref{thm:mi-pi} shows that $\Pi$ is a more informative feature than $X_1$.
However, a model does not necessarily rely more on a more informative feature.
We will discuss more in the next section.

\begin{table*}[]
    \centering
    \begin{tabular}{cc|c|cc|cc}
    \toprule
    & & 1 layer & \multicolumn{2}{c|}{2 layers} & \multicolumn{2}{c}{3 layers} \\
    $d1$ & $\nu$ & w/o & w/o pre & w/ pre & w/o pre & w/ pre \\
    \midrule
     50 & 0.04 &  3680 (189.5) &  691  (55.8) & 614 (169.1) & 302 (47.2) & 249 (53.7) \\
     50 & 0.10 &  2664 (121.2) &  530  (30.6) & 441 (134.9) & 242 (27.6) & 180 (37.5) \\
     50 & 0.25 &  1420  (96.0) &  352  (23.8) & 300  (62.0) & 179 (13.8) & 148 (28.7) \\
     50 & 0.50 &   306  (79.8) &  141  (40.7) & 118  (33.4) & 106 (23.1) &  89 (24.0) \\
    100 & 0.04 &  5466 (170.1) &  945  (57.2) & 689 (225.3) & 431 (51.1) & 275 (72.1) \\
    100 & 0.10 &  3789  (99.2) &  677  (32.2) & 478 (142.9) & 317 (30.3) & 208 (44.3) \\
    100 & 0.25 &  1952  (64.9) &  428  (13.1) & 330  (85.0) & 214 (16.2) & 169 (32.5) \\
    100 & 0.50 &   330  (78.0) &  156  (34.0) & 133  (41.2) & 128 (28.2) & 112 (36.1) \\
    500 & 0.04 & 11127 (265.9) & 1953 (112.5) & 857 (442.6) & 792 (69.8) & 431 (88.4) \\
    500 & 0.10 &  7912 (169.2) & 1279  (67.5) & 657 (234.9) & 550 (46.7) & 402 (97.0) \\
    500 & 0.25 &  4321 (152.3) &  772  (35.5) & 501 (133.5) & 399 (42.3) & 391 (66.0) \\
    500 & 0.50 &   576 (150.0) &  392  (70.2) & 407  (81.1) & 367 (69.1) & 386 (80.0) \\
    \bottomrule
    \end{tabular}
    \caption{The number of iterations required to converge. The number is the average of 25 runs with different random seeds, and the number in parentheses is the standard deviation.}
    \label{tab:toy-exp-ptr}
\end{table*}

\subsection{Learning from $\Pi$ is Easy}
\label{sec:easy}

It is important that learning from $\Pi$ is easy.
Because of simplicity bias, a neural network model is likely to rely on the easy-to-learn features due to the simplicity bias \cite{shah2020pitfalls,kalimeris2019sgd}.
We conjecture that a model excessively relies on the spurious feature $X_1$ when learning from $X_1$ is easier than learning from the robust feature $X_2$.
If learning from $\Pi$ is easy, then the model will rely on $\Pi$ more and thus will rely on $X_1$ less.
However, features with higher mutual information to $Y$ are not necessarily easy to learn.
For instance, models tend to rely on $X_1$ instead of $X_1$ at the beginning of the training process.
MLM can possibly mitigate the issue of spurious features if learning from $\Pi$ is easy.

We show that learning from $\Pi$ is at least as easy as learning from $X_1$ by proving the following theorem:
\begin{theorem}
Let $\tilde{h}^{(D_n)}_{X_1}$ be the classifier trained with MLE loss using $n$ data pairs $(x^{(1)}_1, y^{(1)}), (x^{(2)}_1, y^{(2)}), \cdots (x^{(n)}_1, y^{(n)})$, and the converged classifier be $\tilde{h}^*_{X_1}$.
Given $P(Y)$, there exists a learning algorithm.
It generates $\tilde{h}^{(D_n)}_{\Pi}$ using $(\pi^{(1)}, y^{(1)})$, $(\pi^{(2)}, y^{(2)})$, $\cdots$, $(\pi^{(n)}, y^{(n)})$ and satisfies the following three properties:
(1) $\tilde{h}^{(D_n)}_{\Pi}$ converges asymptotically no slower than $\tilde{h}^{(D_n)}_{X_1}$ does.
(2) Over the distribution of $(X, Y)$, the loss expectation of the converged classifier $\tilde{h}^*_{\Pi}$ is no more the expectation of $\tilde{h}^{*}_{X_1}$.
(3) $\tilde{h}^*_{\Pi}$ is a linear model.
\label{thm:easy}
\end{theorem}
\begin{proof}
Proof sketch: The classifier that maximizes the likelihood of $(x^{(1)}_1, y^{(1)}), (x^{(2)}_1, y^{(2)}), \cdots (x^{(n)}_1, y^{(n)})$ can be attained by counting the co-occurrence of $X_1$ and $Y$. It converges to
\begin{equation}
    \tilde{h}^{*}_{X_1}(y | X_1 = x) = P(y | X_1 = x).
\end{equation}

Based on $\Pi_1$, a classifier can be attained by first estimating $P(Y)$ and $P(x_1 | y)$ for all $x_1$ and $y$:
\begin{equation}
    \rho_y =  \frac{\abs{\left\{ i | y^{(i)} = y \right\}}}{n} \quad
    \rho_{X_1|y}^{(n)} = \frac{\sum_{\{ i | y^{(i)} = y \}} \pi^{(i)}}{\abs{\left\{ i | y^{(i)} = y \right\}}},
\end{equation}
and then we can construct a classifier that converges to
\begin{equation}
    \tilde{h}^*_{\Pi}(y | \pi) = \sum_{x_1} P(y | x_1) \pi.
\end{equation}
For both $\tilde{h}^{*}_{X_2}$ and $\tilde{h}^*_{\Pi}(y | \pi)$ the $l_2$ error converges to $0$ at a rate of $O(\frac{1}{n})$.

Then we show that $\tilde{h}^*_{\Pi}(y | \pi)$ is at least as good as $\tilde{h}^*_{X_2}(y | \pi)$ by showing $\kld{P(Y|X)}{\tilde{h}^{*}_{X_1}(Y | X_1, X_2, \Pi)} \geq \kld{P(Y|X)}{\tilde{h}^{*}_{\Pi}(Y | X)}$ with convexity:
\begin{align}
    \sum_{x_1} P(x_1 | x_2) \kld{P(Y|x_2)}{P(Y|x_1)} \geq \kld{P(Y|x_2)}{\sum_{x_1} P(Y | x_1) P(x_1 | x_2)}.
\end{align}

\end{proof}

The remaining question is whether a deep learning model used in common practices can perform at least as well as the algorithm in Theorem~\ref{thm:easy}.
Indeed, without any knowledge of deep learning models, it is impossible to theoretically prove that a model will necessarily rely on $\Pi$ instead of $X_1$.
Therefore, in Section~\ref{sec:toy-example-ptr} and Section~\ref{sec:experiments} we will empirically validate that our theorems are applicable in the real world scenarios.

\subsection{Extending with Causal Models}
\label{sec:causal}

Before we validate our theories with empiricially, we make a step further by relaxing Assumption~\ref{assumption:deterministic}.
We can do so by treating $X_2$ as a confounder, and then we can see how MLM pre-training is helpful in the causal and anticausal settings.

\begin{theorem}
Even if Assumption~\ref{assumption:deterministic} is not true, Theorem~\ref{thm:mi-pi} still holds if $X_1, X_2, Y$ follow the causal setting in Figure~\ref{fig:causal-graph}.
\end{theorem}
\begin{proof}
By the structure of $X_1, X_2, Y$, inequality~\ref{eq:mi-pi-data-processing} holds even if the deterministic mapping $f_{X_2 \to Y}$ does not exist.
\end{proof}

\begin{theorem}
Assume that the set of vectors $\{P(X_1 | Y = y) | y \in \mathcal{Y} \}$ is linear independent, and if $X_1, X_2, Y$ follow the anticausal setting in Figure~\ref{fig:anticausal-graph}, then $I(\Pi, Y) = I(X_2, Y)$.
\label{thm:recover-anticausal}
\end{theorem}
\begin{proof}
The assumption is a special case of the one in \cite{lee2020predicting}, so similar techniques can be used: According to the structure of $X_1, X_2, Y$, we have
\begin{align}
     P(X_1 | X_2) = \sum_{y} P(X_1 | y) P(y | X_2) .
\end{align}
Therefore, if $\{P(X_1 | Y = y) | y \in \mathcal{Y} \}$ is linearly independent, $P(y | X_2)$ can be recovered from $\Pi = P(X_1 | X_2)$.
\end{proof}

Note that this theorem is very similar to Theorem~3.3 in \cite{wei2021pretrained}.
However, the assumptions required in ours are weaker and more realistic, and the implication is very similar.
The assumption in \cite{wei2021pretrained} that $X$ is generated from a HMM process with hidden variables $H_0, H_1, \cdots$ is a assumption stronger than our assumption that $X_1, X_2$ follows the anticausal setting.
The assumption in \cite{wei2021pretrained} that the vectors in $\{P(X_0 | H_0 = h) | h \in \mathcal{H} \}$ need to be linearly independent is also less realistic than our independence assumption that requires only the independence in $\{P(X_1 | Y = y) | y \in \mathcal{Y} \}$, because the number of hidden states $\lvert \mathcal{H} \rvert$ must be very large if $X$ is generated from the HMM model.
However, $\lvert \mathcal{Y} \rvert$ tends to be much smaller.
For binary classification cases, our assumption holds as long as $P(X_1)$ is not independent of $P(Y)$.
If we further assume that $I(X_2; Y) = I(X; Y)$, then we reach a similar conclusion that $P(Y | X)$ can be recovered from $\Pi = P(X_1 | X_2)$ by applying a linear function.

\begin{figure*}
\includegraphics[width=0.24\linewidth]{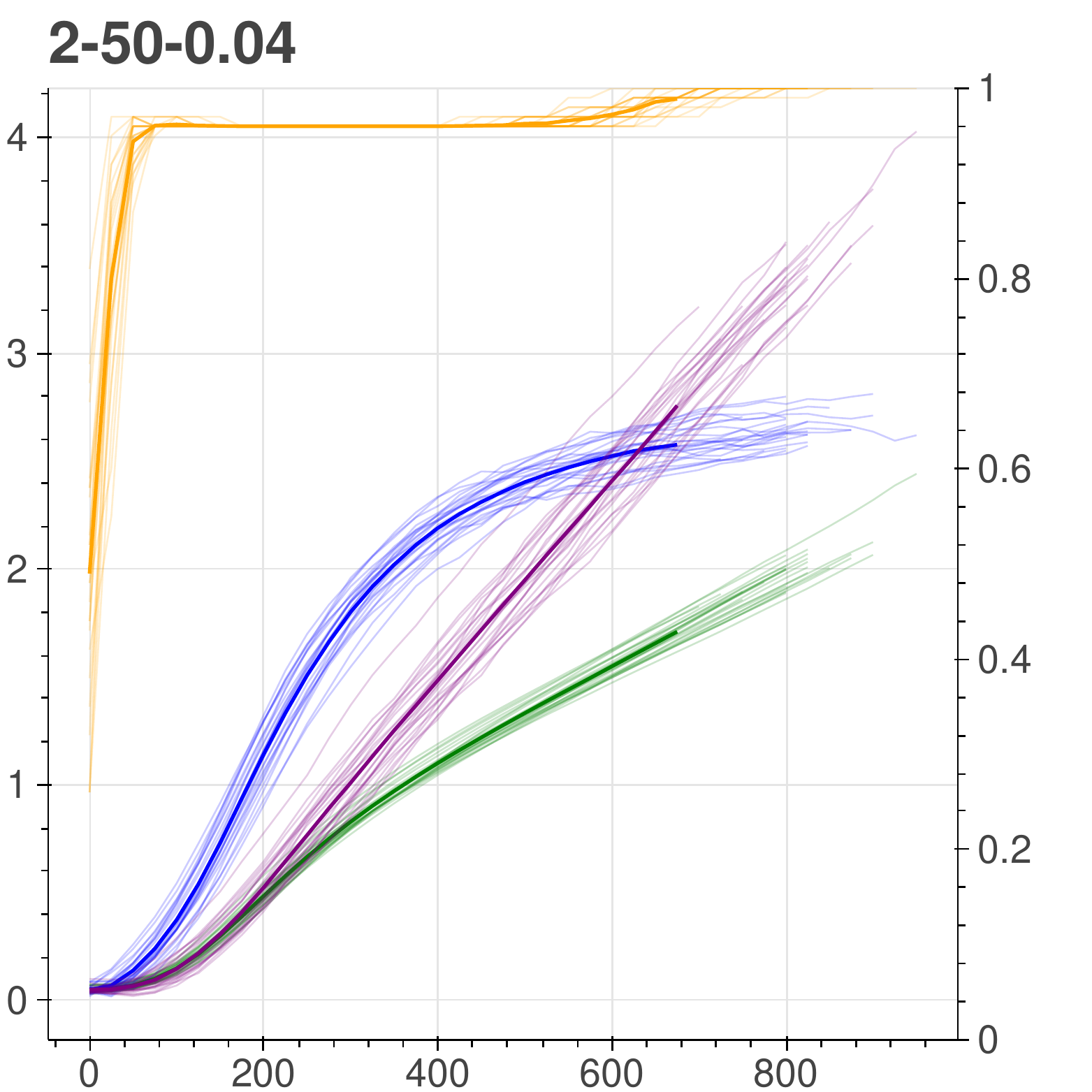}
\includegraphics[width=0.24\linewidth]{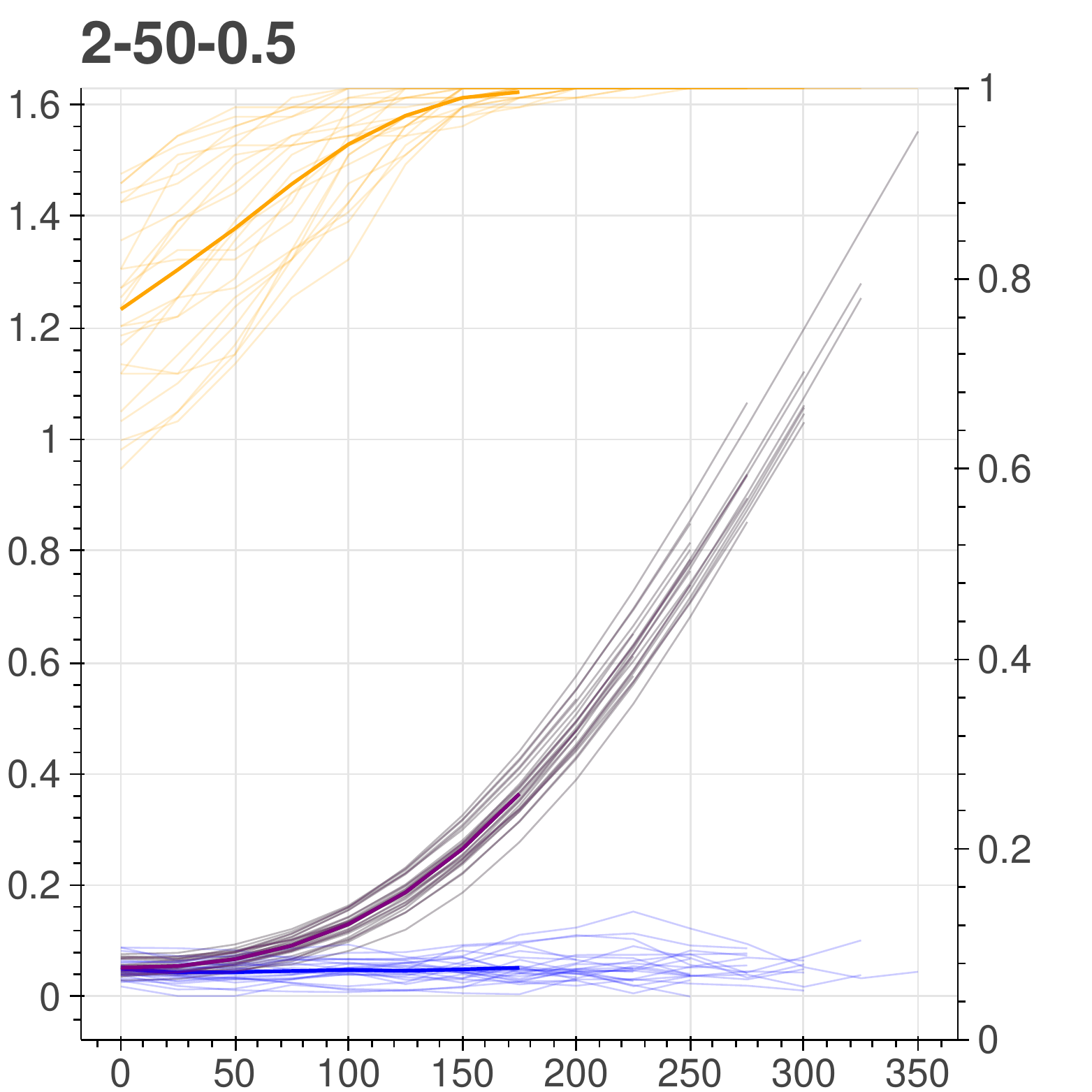}
\includegraphics[width=0.24\linewidth]{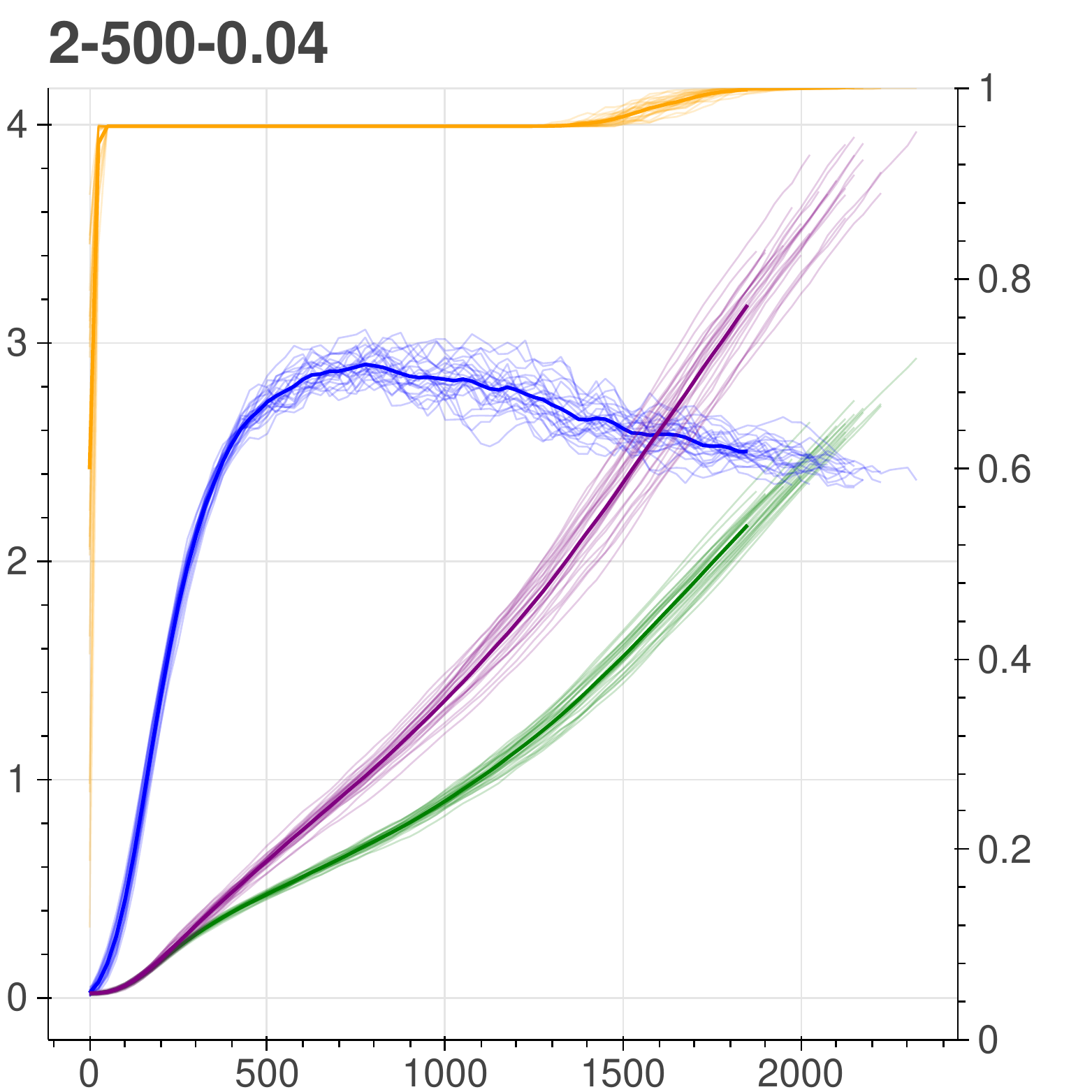}
\includegraphics[width=0.24\linewidth]{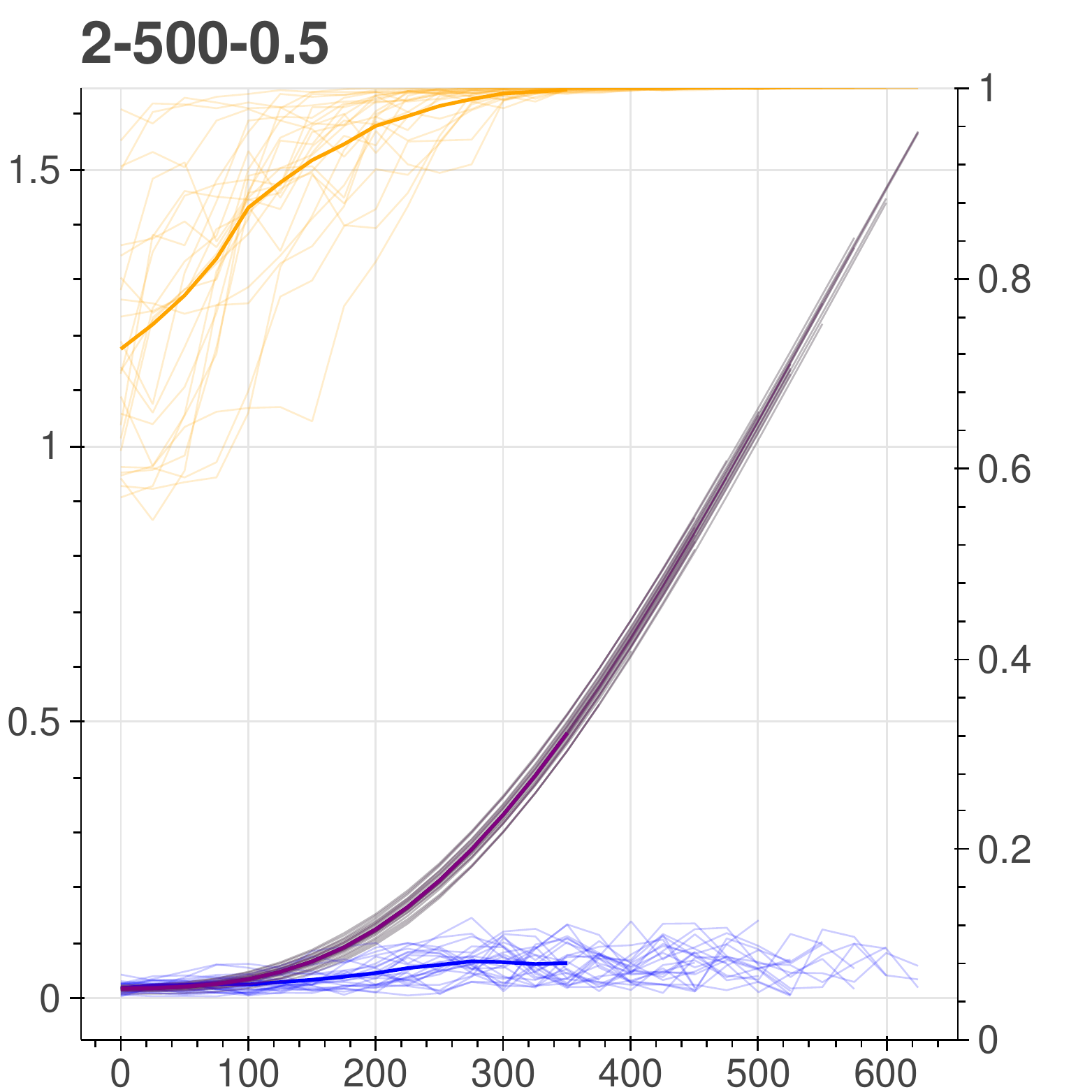}

\includegraphics[width=0.24\linewidth]{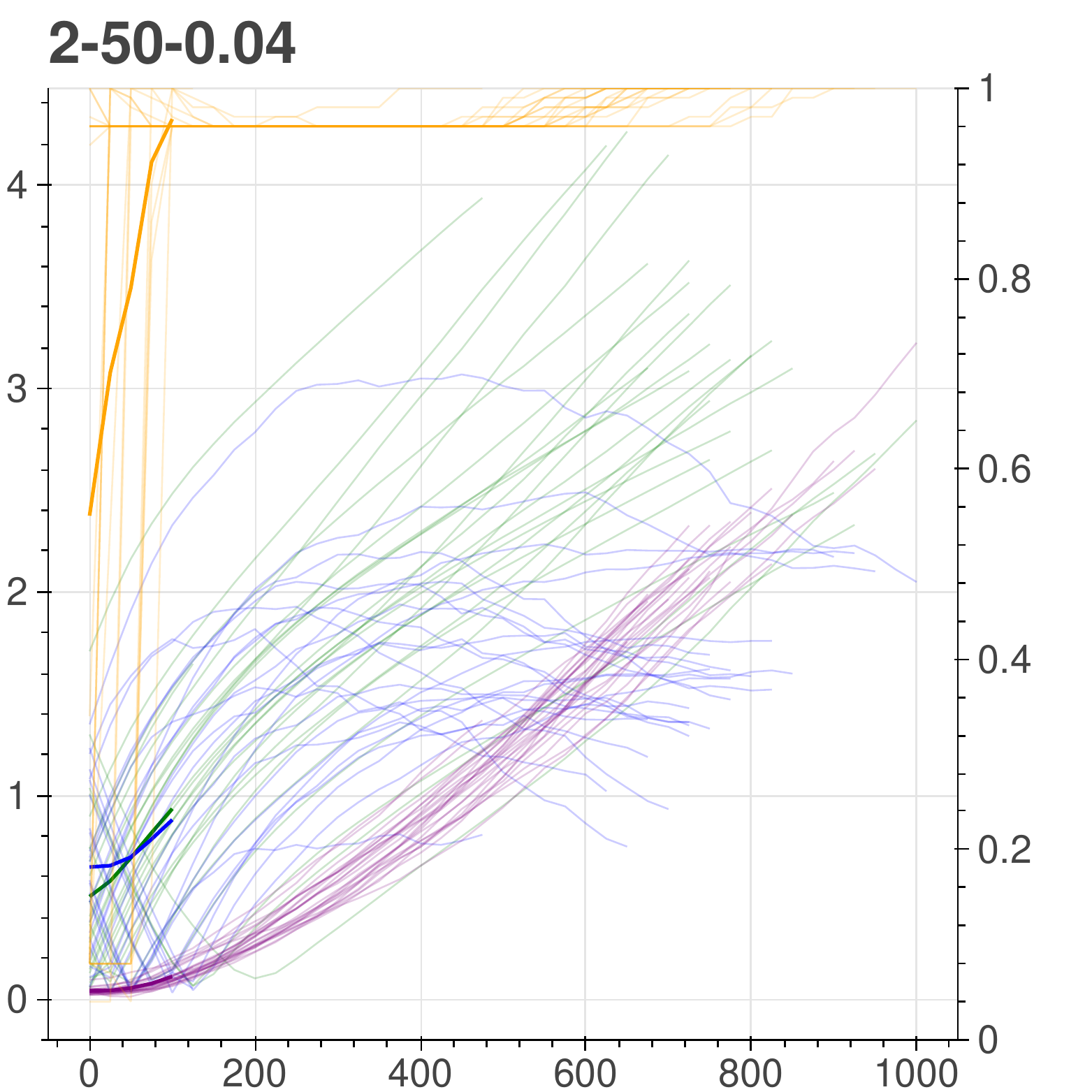}
\includegraphics[width=0.24\linewidth]{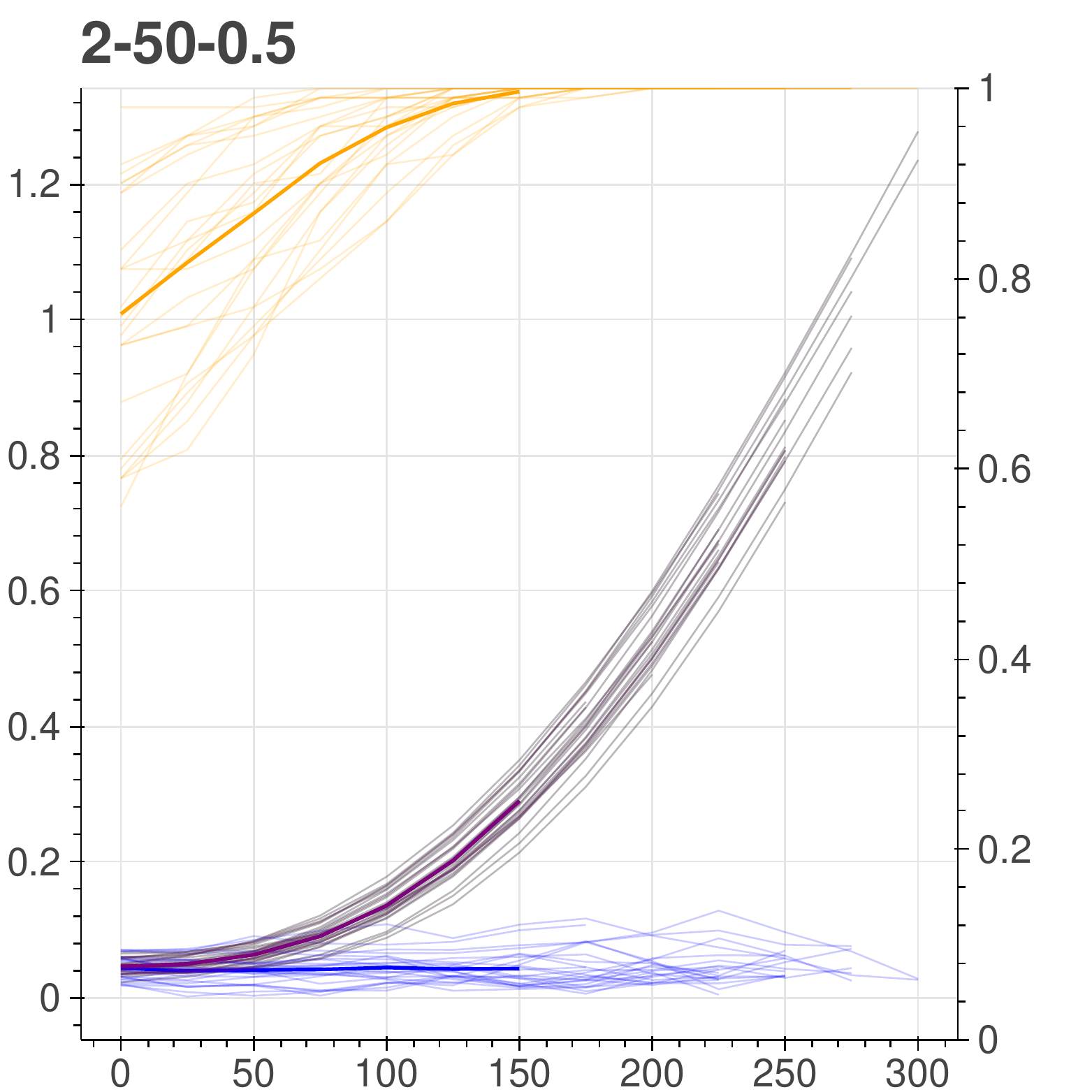}
\includegraphics[width=0.24\linewidth]{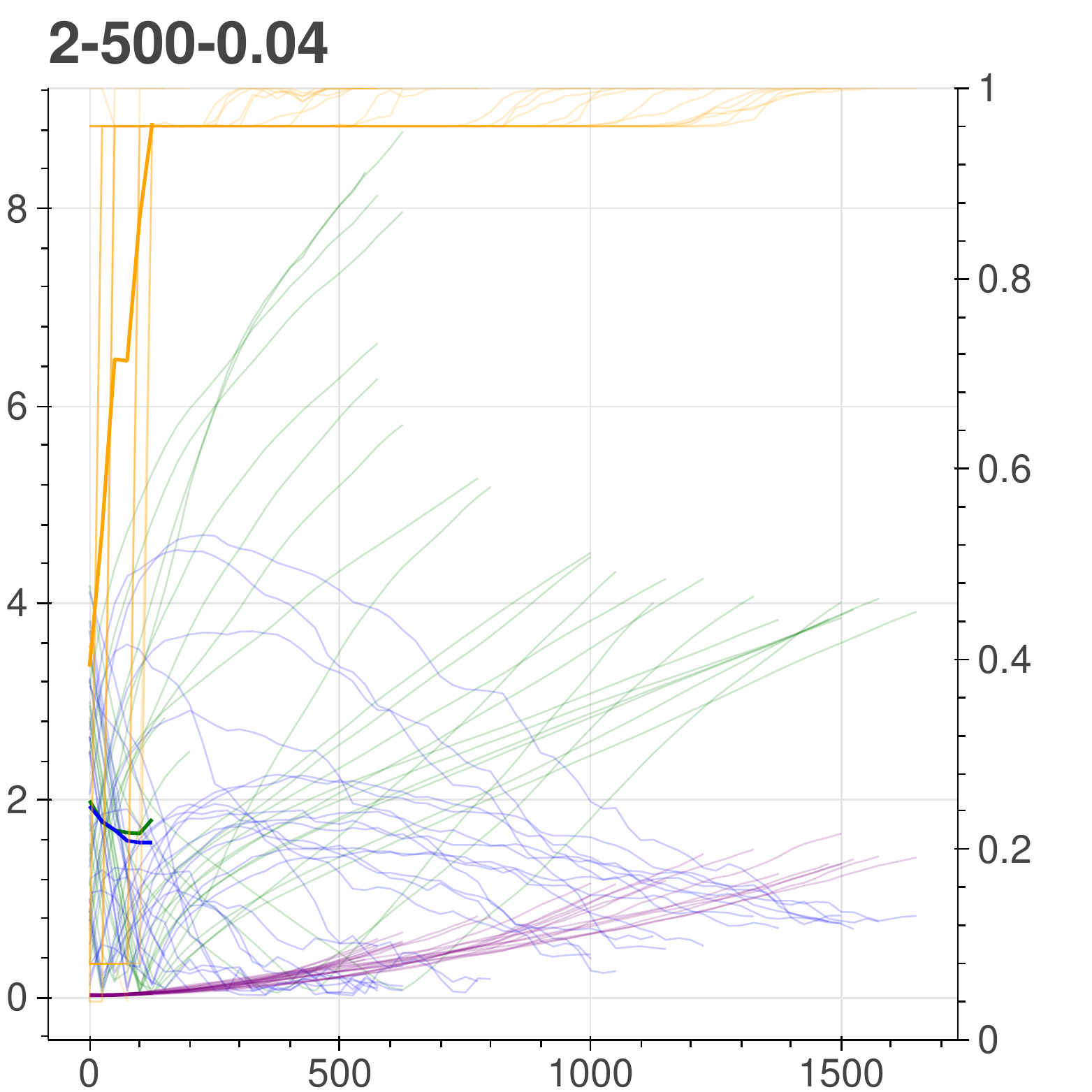}
\includegraphics[width=0.24\linewidth]{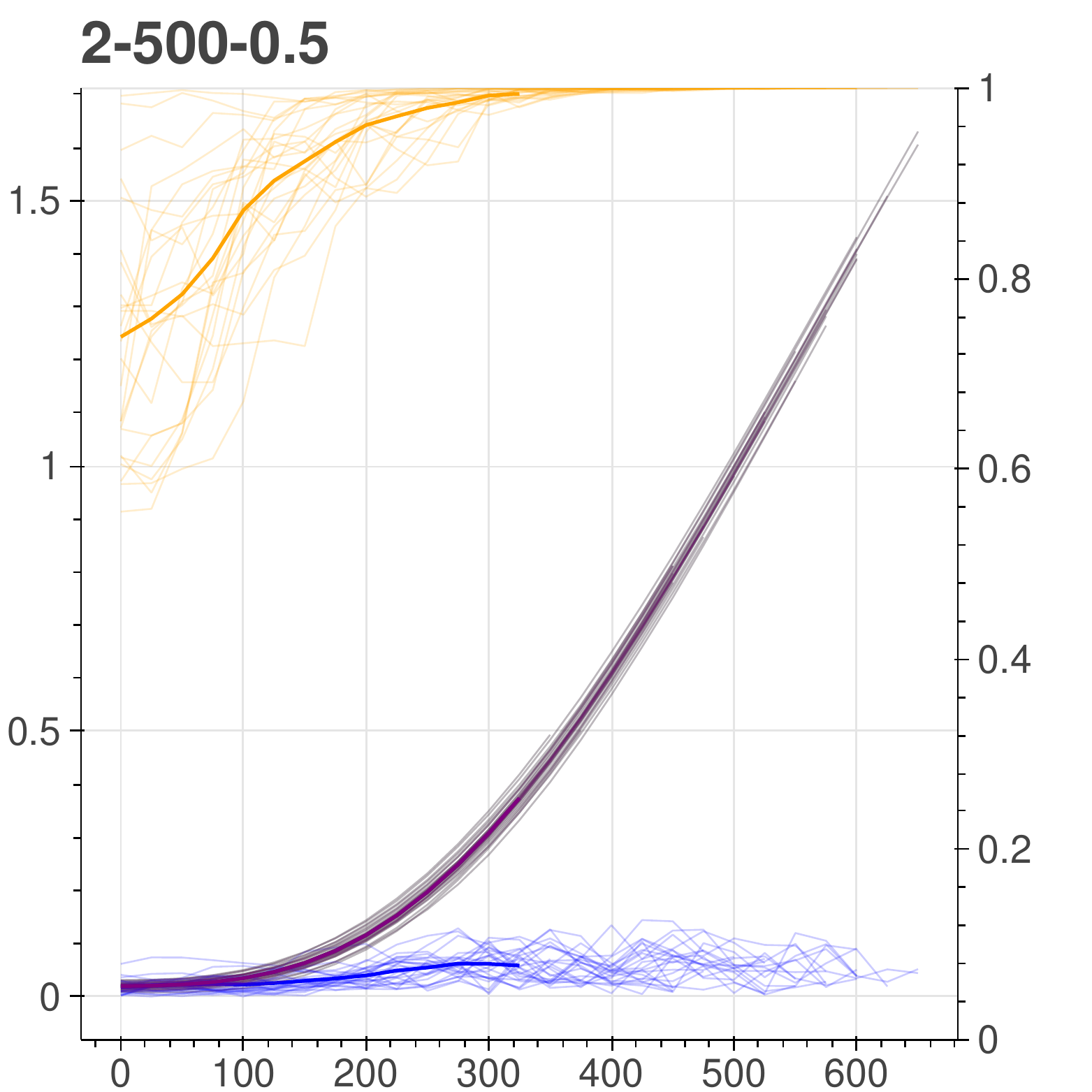}

\caption{The average weights over the features during training a two-layer model.
The upper and the lower rows are the curves of model without and with pretraining respectively.
From left to right, $(d_1, \nu) = (50, 0.04), (50, 0.5), (500, 0.04), (500, 0.5)$.
Blue, green, purple curves represent the average weights over features in \textcolor{blue}{$X_1$}, \textcolor{DarkGreen}{$X_2$}, and \textcolor{purple}{$\dot{X}_2$} (the middle part of $X_2$) respectively. The orange curve represents the \textcolor{orange}{accuracy}. Note that the scales of the $x$-axis are different in the plots.}
\label{fig:toy-curves}
\end{figure*}

\subsection{Limitations of Our Theorems}
Our theories do not ensure that $\Pi_1$ is the most informative feature to learn.
Consider tokens in a sentence $X = \langle X_1, X_2, \cdots, X_L \rangle$ and let $\Pi_i$ be the conditional probability $P(X_i | X \setminus X_i)$.
A token with spurious association with the label can locate arbitrary position in the sentence, and its location is unknown during pretraining.
That is, the pretrained model is able to generate $\Pi_i$ for all $i$.
Without loss of generality, assume $X_1$ is the spurious token.
\textit{It is possible that there exists some $i$ such that $I(\Pi_1; Y) < I(\Pi_i; Y)$, and that $\Pi_i$ is predicted relying on $X_1$.}
Concretely, here is an example for the causal setting with three features:
$X_3$ is independent of $X_1$ and $X_2$ given $Y$ (Figure~\ref{fig:special-case}).
Using the results in Theorem~\ref{thm:recover-anticausal}, there is a linear mapping that can recover $P(Y | X_1, X_2)$ from $\Pi_3$.
Therefore, it is possible that $I(\Pi_3; Y) > I(\Pi_1; Y)$ if $I(X_1, X_2; Y) > I(\Pi_1; Y)$ depending on the distribution of the data.
We leave the study of $I(\Pi_i; Y)$ for future work.

Another limitation is that, in practice, NLP practitioners do not use the conditional probability predicted by the pretrained model.
Instead, people stack a simple layer over the pretrained model, and fine-tune the whole model on downstream task.
Regardless of this, we conjecture that the outputs generated by feeding the whole sentence without masking also contain the information of $\{ \Pi_{i} \}_{i=1}^{n}$ and thus are robust to spurious token-level feature.

\section{Toy Example with a Pretrained Model}
\label{sec:toy-example-ptr}

As the first step to close the gap between our theories and real world, we pretrain the experiment with the toy example.
We use the two-layer and three-layer MLP architectures same as in Section~\ref{sec:toy-example}.
Before fitting the model with $Y$, we first pretrain the first layer to predict $X_1$ based on $X_2$.
Specifically, we mask $X_1$ when pretraining, so the inputs are $[0, 0; X_2]$, and the first two dimensions of the outputs are used to compute cross entropy loss.
After pretraining, we conjecture that the first two dimensions of the outputs have the equivalent role of $\Pi_1$.
In order to allow the information from $X_1$ to compete with the output of the pretrained part fairly, we manually initialize the weights of the third and fourth dimension of the outputs with $[k, -k, 0, \cdots, 0]$ and $[-k, k, \cdots, 0, 0]$ respectively, where $k$ is the average of the absolute value of the weights in the pretrained part.
This ensures that the scales of the first four dimensions of the outputs are the same before fine-tuning it with the label $Y$.
Finally, we fine-tune the pretrained model with $([X_1; X_2], Y)$ pairs, and report the average number of iterations required to converge for 25 different random seeds.

Table~\ref{tab:toy-exp-ptr} shows that pretraining can always reduce the number of iterations required to converge, when $\nu < 0.50$.
The effect is more significant when $d_2$ equals 100 and 500.
It could be because of the higher sparsity of the learning signal.

We further inspect how the product of the weights in the layers change in the process of training.
We observe that if the model is not pretrained, the \textcolor{blue}{weights over $X_1$} grow faster than the \textcolor{DarkGreen}{weights over $X_2$} at the beginning (the first row Figure~\ref{fig:toy-curves}).
The model cannot converge to 100\% accuracy until \textcolor{purple}{weights on $\dot{X}_2$}, the central $\lceil \nu \times d_2 \rceil$ dimensions of $X_2$, become greater than the \textcolor{blue}{weights on $X_1$}.
In addition, after the model converges, \textcolor{blue}{weights over $X_1$} is still greater than \textcolor{DarkGreen}{weights over $X_2$}.
On the other hand, if the model is pretrained, \textcolor{blue}{weights over $X_1$} stop growing after a few steps (the second row in Figure~\ref{fig:toy-curves}).
The above observations are aligned with our theorems that pretraining help the model avoid the simplicity bias.

\begin{figure}
    \center
    \subfloat[Causal setting]{
        \begin{minipage}[c][0.75\width]{0.25\textwidth}
        \small
        \center
        \begin{tikzpicture}[node distance={10mm}, minimum size={7.5mm}, thick, main/.style = {draw, circle}]
            \node[main] (1) {$Z$};
            \node[main] [right of=1] (2) {$X_1$};
            \node[main] [above of=2] (3) {$X_2$};
            \node[main] [right of=2] (4) {$Y$};
            \draw[->] (1) -- (2);
            \draw[->] (1) -- (3);
            \draw[->] (2) -- (4);
        \end{tikzpicture}
        \end{minipage}
        \label{fig:causal-graph}
    }
    \subfloat[Anticausal setting]{
        \begin{minipage}[c][0.75\width]{0.25\textwidth}
            \center
            \small
            \begin{tikzpicture}[node distance={10mm}, minimum size={7.5mm}, thick, main/.style = {draw, circle}]
                \node[main] (1) {$Z$};
                \node[main] [right of=1] (2) {$Y$};
                \node[main] [above of=2] (3) {$Q$};
                \node[main] [right of=2] (4) {$X_1$};
                \node[main] [right of=3] (5) {$X_2$};
                \draw[->] (1) -- (2);
                \draw[->] (1) -- (3);
                \draw[->] (2) -- (4);
                \draw[->] (3) -- (5);
            \end{tikzpicture}
        \end{minipage}
        \label{fig:anticausal-graph}
    }   
   \subfloat[A case where $I(\Pi_3; Y) \geq I(\Pi_1; Y)$ is possible.]{
        \begin{minipage}[c]{0.25\textwidth}
        \small
        \center
        \begin{tikzpicture}[node distance={10mm}, minimum size={7.5mm}, thick, main/.style = {draw, circle}]
            \node[main] (1) {$Z$};
            \node[main] [right of=1] (2) {$X_2$};
            \node[main] [above of=2] (3) {$X_1$};
            \node[main] [right of=2] (4) {$Y$};
            \node[main] [below of=2] (5) {$X_3$};
            \draw[->] (1) -- (2);
            \draw[->] (1) -- (3);
            \draw[->] (2) -- (4);
            \draw[->] (5) -- (4);
        \end{tikzpicture}
        \end{minipage}
        \label{fig:special-case}
    }
    \caption{The causal settings of the $(X, Y)$ pairs.}
\end{figure}
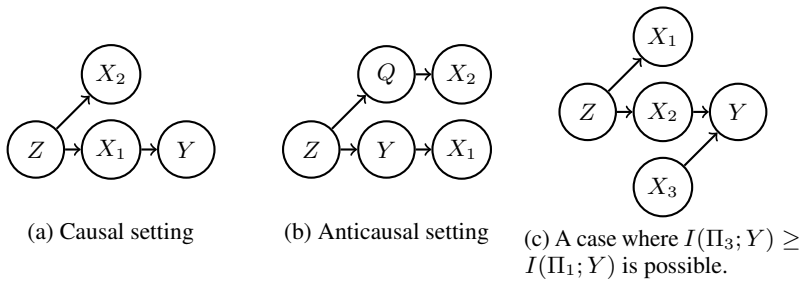

\section{Experiments}
\label{sec:experiments}

To further validate our theories, we experiment on real world NLP tasks.
We facilitate datasets that have known spurious features.
We first pretrain models on the training dataset with different MLM settings.
In different settings, the probabilities of masking the spurious tokens are different.
Afterward, we fine-tune the model using the target label.
We show that the models will be less robust on downstream tasks if spurious tokens are not masked during pretraining,
and always masking the spurious token during pretraining improving the robstness.
The results validate our theories.

\subsection{Dowstream Tasks}

\paragraph{Hate Speech Detection}
Previous study has shown that hate speech detection datasets tend to have lexical bias \cite{dixon2018measuring}.
That is, models rely excessively on the presence or the absence of certain words when predicting the label.
Here we follow the formulation of lexical bias in hate speech detection proposed by \citet{zhou-etal-2021-challenges}.
We focus on the effect of non-offensive minority identity (NOI) mentions, such as "woman", "gay", "black".
Those mentions are often highly associative with hateful instances.
However, it is more desirable that a model does not rely on those mentions.
Therefore, we can see the presence of NOI as a spurious feature.

\paragraph{Name Entity Recognition (NER)} \citet{lin-etal-2020-rigorous} has shown that name entity recognition (NER) models perform worse when the name entities are not seen in the training data.
In this case, the content of the name entities can be seen as a spurious feature.
Models may learn to memorize the content when fitting the training data, while we may desire the model to recognize name entities according to the context.

\subsubsection{Datasets}

\paragraph{Hate Speech Detection} We use a portion of the dataset proposed by \citet{zhou-etal-2021-challenges}.
In their original dataset, only a small number of hateful instances contain NOI.
Our preliminary experiments show that model without pretraining does not suffer much from the bias of NOI when using the full data.
Therefore, we create a dataset of which positive (hateful) samples are all the positive samples in the original dataset that contain NOI, negative samples are negative samples randomly sampled from the original training set.
We control the number of negative samples so the ratio of positive and negative samples is the same as the original dataset.
We create both the training and the validation splits in this way, and use the original full testing set for evaluation.
We also evaluate the models on a NOI subset where all the instances contain NOI.

\paragraph{NER} We use the standard NER dataset Conll-2003~\cite{tjong-kim-sang-de-meulder-2003-introduction}.
To create a testing set with unseen name entities, we replace the name entities in the original validation and testing splits with the entities from WNUT-17~\cite{derczynski-etal-2017-results}.
Specifically, we replace the \texttt{LOC}, \texttt{ORG}, \texttt{PER} entities with the corresponding type of entities in WNUT-17, while the \texttt{MISC} entities remain untouched.

\subsection{Masking Policies}

For each sentence with $n_s$ spurious tokens, we experiment with different masking policies:
\textbf{(1) scratch}: We do not pretrain the model before fine-tuning.
\textbf{(2) vanilla}: During pretraining, we mask each token with 15\% probability, which is same as the original setting in \cite{devlin-etal-2019-bert}.
\textbf{(3) unmask random}: This is similar to vanilla MLM, but we uniformly randomly select $n_s$ tokens from the whole sentence and unmask them if they have been masked.
\textbf{(4) unmask spurious}: This is similar to vanilla MLM, but we unmask all the spurious tokens.
\textbf{(5) unknown spurious}: We replace spurious tokens with a special "[unk]" token, and we unmask them. Note that this setting can be seen as an oracle setting, since in most applications the spurious features are unknown.

Note that setting (3), (4), (5) have the same expected number of masked tokens.
Therefore, it rules out the possibility that the difference between their downstream performance is due to the number of masked tokens.
Instead, the difference should be resulted from whether spurious tokens are masked.

\begin{table}
\centering
\begin{tabular}{l|cc|cc|cc}
\toprule
\multirow{3}{*}{Mask Policy}       & \multicolumn{2}{c|}{NER}                                         & \multicolumn{4}{c}{Hate Speech Detection}                                                                                        \\
 & \multicolumn{1}{c}{Origin}     & \multicolumn{1}{c|}{Unseen}  & \multicolumn{2}{c|}{All (12893)}                                & \multicolumn{2}{c}{NOI (602)}                                   \\
                              & \multicolumn{1}{c}{F1 $\uparrow$}         & \multicolumn{1}{c|}{F1 $\uparrow$}        & \multicolumn{1}{c}{Accuracy $\uparrow$}   & \multicolumn{1}{c|}{F1 $\uparrow$}         & \multicolumn{1}{c}{Accuracy $\uparrow$}   & \multicolumn{1}{c}{FPR $\downarrow$}  \\
                              \midrule
scratch         & 61.5 \sd{0.5} & 38.7 \sd{0.6} & 83.9 \sd{1.6} & 80.3 \sd{1.4} & 74.8 \sd{1.5} & 46.3 \sd{7.2} \\
random          & 74.2 \sd{0.4} & 56.5 \sd{1.3} & 83.1 \sd{0.8} & 78.5 \sd{0.8} & 75.8 \sd{0.5} & 25.1 \sd{1.8} \\
unmask random   & 72.7 \sd{0.6} & 56.5 \sd{0.8} & 83.3 \sd{1.1} & 78.9 \sd{1.1} & 75.8 \sd{0.9} & 25.7 \sd{2.3} \\
unmask spurious & 72.9 \sd{0.5} & 53.2 \sd{0.8} & 84.1 \sd{0.7} & 79.8 \sd{0.6} & 73.7 \sd{1.0} & 32.5 \sd{2.1} \\
unkown spurious & 69.8 \sd{0.5} & 56.7 \sd{0.8} & 82.4 \sd{1.0} & 77.8 \sd{1.0} & 77.3 \sd{0.6} & 21.7 \sd{2.0} \\
\bottomrule
\end{tabular}
\caption{The performance on downstream tasks. For the hate speech detection task, we also report false positive detection (FPR) on the NOI subset, which is a set of instances containing non-offensive minority identity mentions, e.g. "women", "black". }
\label{table:results}
\end{table}

\subsection{Implementation Details}

For both of the tasks and all the MLM settings, we tokenize the input with the bert-base-uncased tokenizer.
We reinitialize all layers except the embedding layer before pretraining, and the embedding layer is frozen through the pretraining process.
We train the models until they converge, and choose the one with the lowest MLM loss on the validation set.
For the hate speech detection task, we use the implementation provided by \citet{zhou-etal-2021-challenges}. Except that we use bert-base-uncased instead of roberta-large, we use the other hyper parameters provided in their script. For the NER task, we use the implementation by Hugging Face \footnote{https://github.com/huggingface/transformers/blob/master/examples/pytorch/token-classification/run\_ner.py}.

\subsection{Result and Discussion}
Results in Table~\ref{table:results} validate our theorems.
For both of the tasks, \textit{unmask random} performs better than \textit{unmask spurious}.
Specifically, \textit{unmask random} has higher F1 on the unseen set of the NER task, and \textit{unmask random} has low false positive rate (FPR) on the NOI set.
Also, \textit{unmask random} performs similarly to \textit{random}.
This implies that modeling the condition distribution of spurious tokens in the original random masking pretraining can reduce models' reliance on them.
Note that \textit{unmask random} and \textit{unmask spurious} have similar in-distribution performance, so the performance difference is not due to better in-distribution generalization suggested by \citet{miller2021accuracy}.

We also compare \textit{unmask random} with the oracle setting \textit{unknown spurious}. 
We notice that even though \textit{unknown spurious} performs as well as \textit{random}, \textit{unknown spurious} hurts the performance in the seen set.
It indicates that modeling the conditional distribution of spurious tokens has effects beyond simply removing them from the model.
On the other hand, \textit{unknown spurious} performs better in the hate speech detection task.
A possible explanation is that NOI mentions contain little useful information for the task.

\section{Related Work}

Recently, there are efforts attempting to explain the effectiveness of massive language modeling pretraining.
Theoretically, \citet{saunshi2021a} explores why auto-regressive language models help solve downstream tasks.
However, their explanation is based on the assumption that the downstream tasks are \textit{natural tasks}, i.e. tasks that can be reformulated as sentence completion tasks.
Their explanation also requires the pretrain language model to perform well for any sentence completion tasks, which is not likely to be true in the real world.
\citet{wei2021pretrained} analyzes the effect of fine-tuning a pretrained MLM model.
Nonetheless, they assume that sentences in natural language are generated by an HMM process, and that the latent state distribution can be recovered from the conditional probability of a token given the context.
With a milder assumption, in this work we achieve similar results.
\citet{aghajanyan2020intrinsic} shows that pretrained models have lower intrinsic dimension, and provides a generalization bound based on \citet{arora2018stronger}.
However, why pretrained models have lower intrinsic dimension is still unknown.
Empirically, \citet{zhang-hashimoto-2021-inductive} shows that the effectiveness of MLM pretraining cannot be explained by formulating the downstream tasks as sentence completion problems.
\citet{sinha2021masked} finds evidence supporting the hypothesis that masked language models benefit from modeling high-order word co-occurrence instead of word order.
There are also some theories explaining the effective pretraining not specific to MLM \citet{lee2020predicting,pmlr-v97-saunshi19a,zhang-stratos-2021-understanding}.

\section{Implication}
Our results provide possible explanations for some common practices that are found effective empirically.
First, it could explain why continue pretraining on target dataset is useful.
It may be because continue pretraining allows the MLM model to better model the distribution of spurious features in the target dataset.
Thus the model can better avoid the simplicity pitfall and converges to a better solution.
Second, it provides reasons for more complex masking policies, such as masking a multi-token proper noun at the same time.
It may improve the robustness to spurious features that contain more than one token.
Third, if MLM can alleviate the simplicity bias and help the model to achieve a greater margin, it may also imply that the model has wider optima. This explains the finding in~\citet{hao-etal-2019-visualizing}.

\section{Conclusion}
In this work, we show an advantage of MLM pretraining, which partly explains its efficacy.
We first show the issue of simplicity bias when the input is discrete, and then theoretically and empirically prove that MLM pretraining can alleviate it.
Finally, our experiments on real world data also support our theories.
Our results shed light on future research on robustness and self-supervised learning.

\bibliography{iclr2022_conference,anthology}
\bibliographystyle{iclr2022_conference}

\appendix
\section{Appendix}
\begin{figure*}
\includegraphics[width=0.24\linewidth]{figures/wo-ptr/2-50-0.04.pdf}
\includegraphics[width=0.24\linewidth]{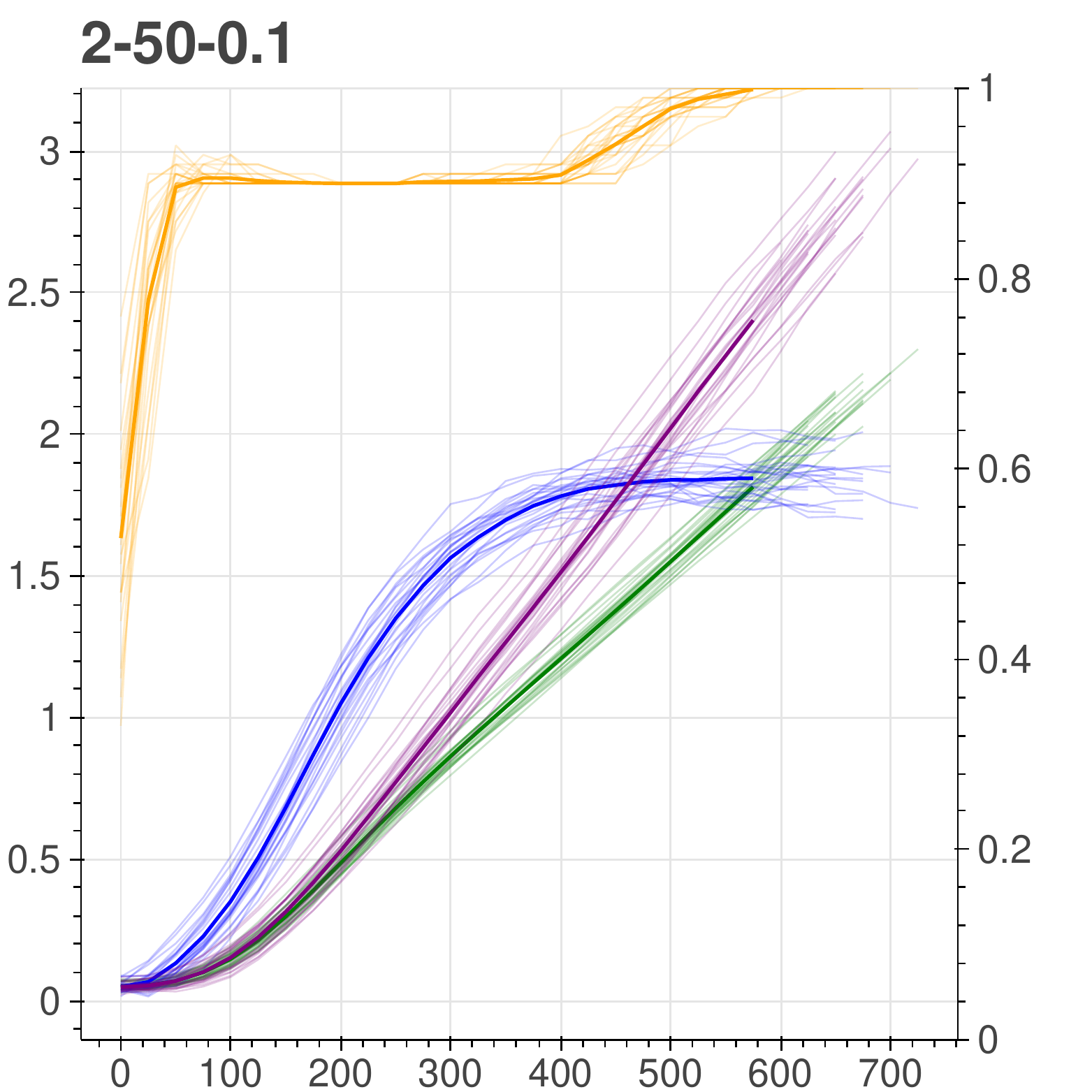}
\includegraphics[width=0.24\linewidth]{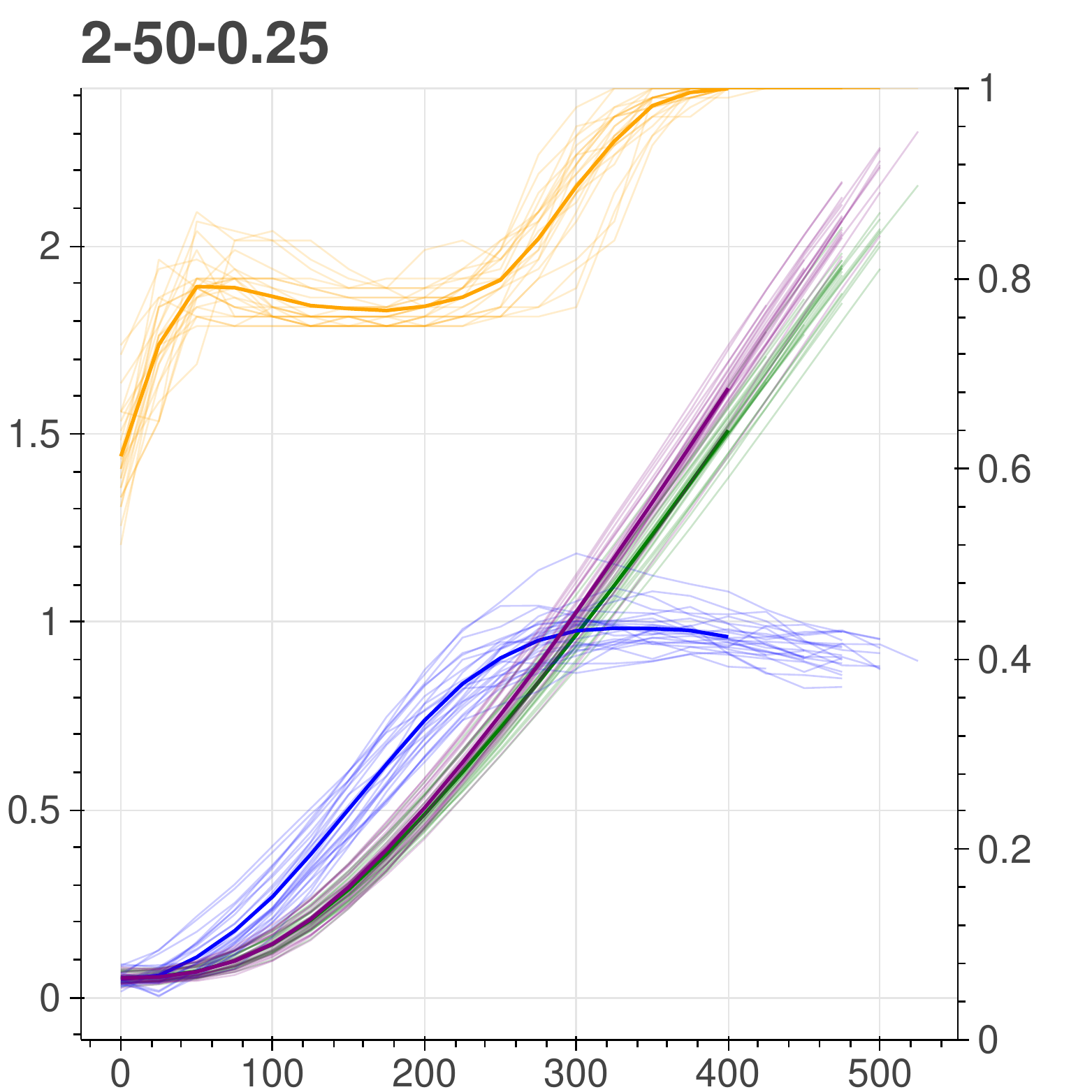}
\includegraphics[width=0.24\linewidth]{figures/wo-ptr/2-50-0.5.pdf}

\includegraphics[width=0.24\linewidth]{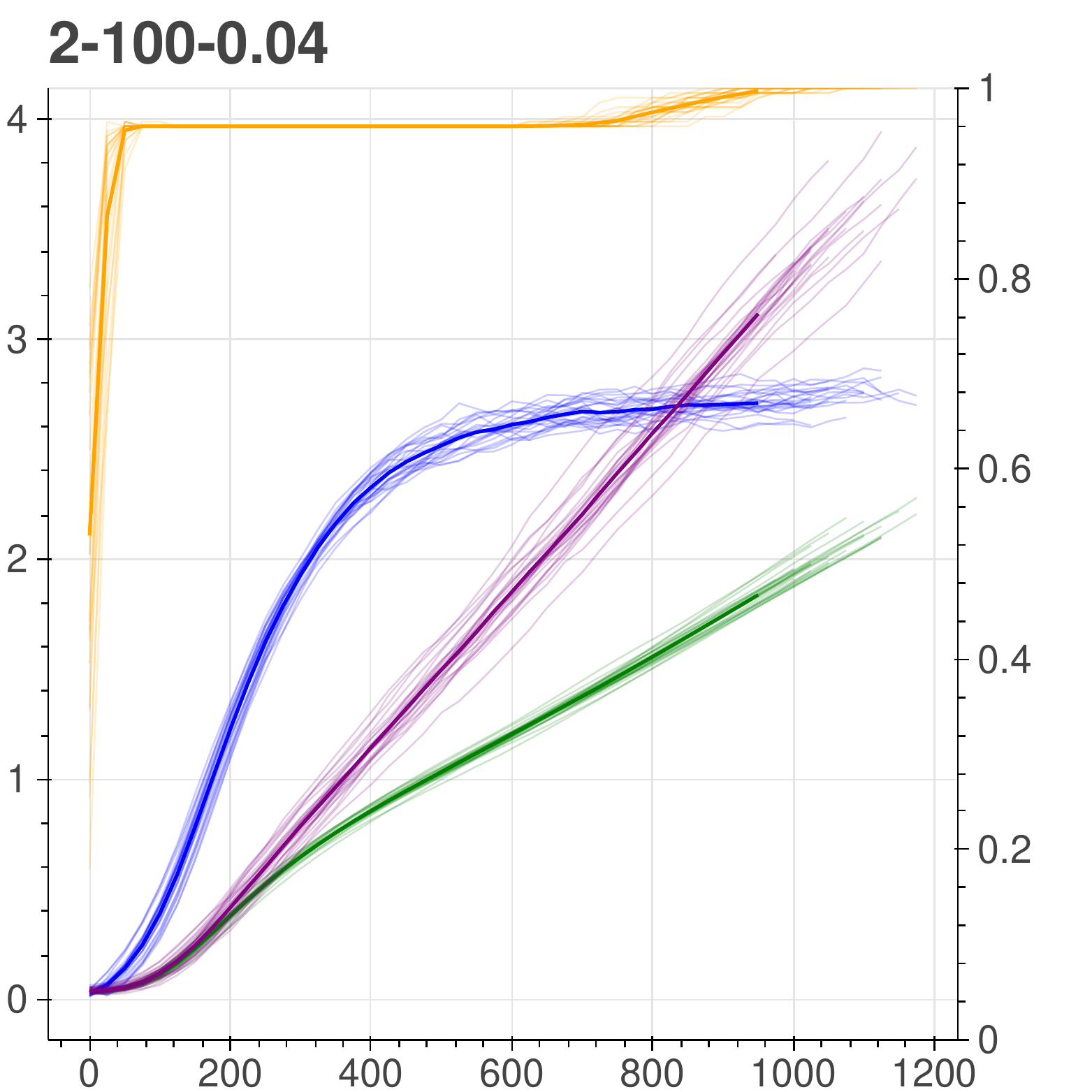}
\includegraphics[width=0.24\linewidth]{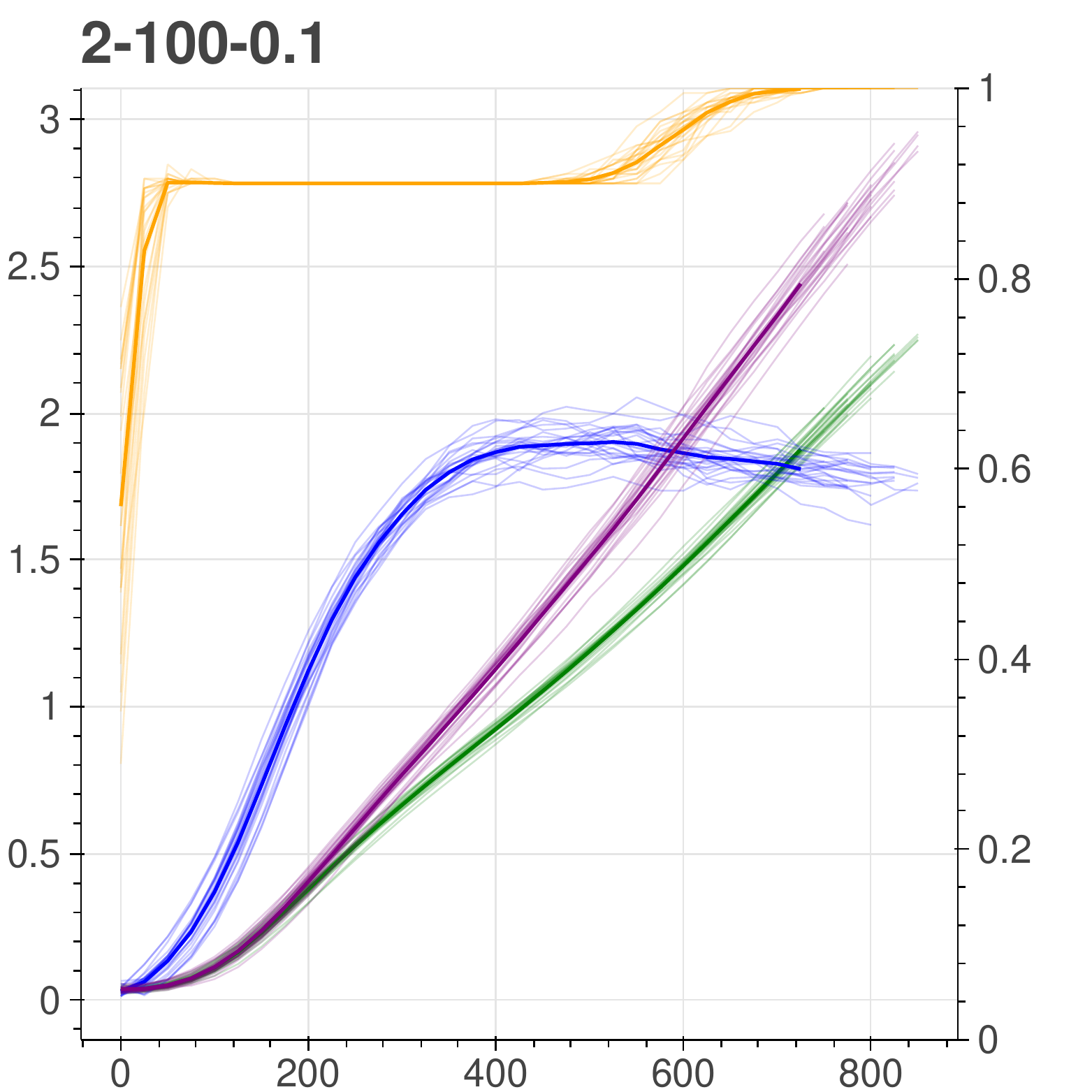}
\includegraphics[width=0.24\linewidth]{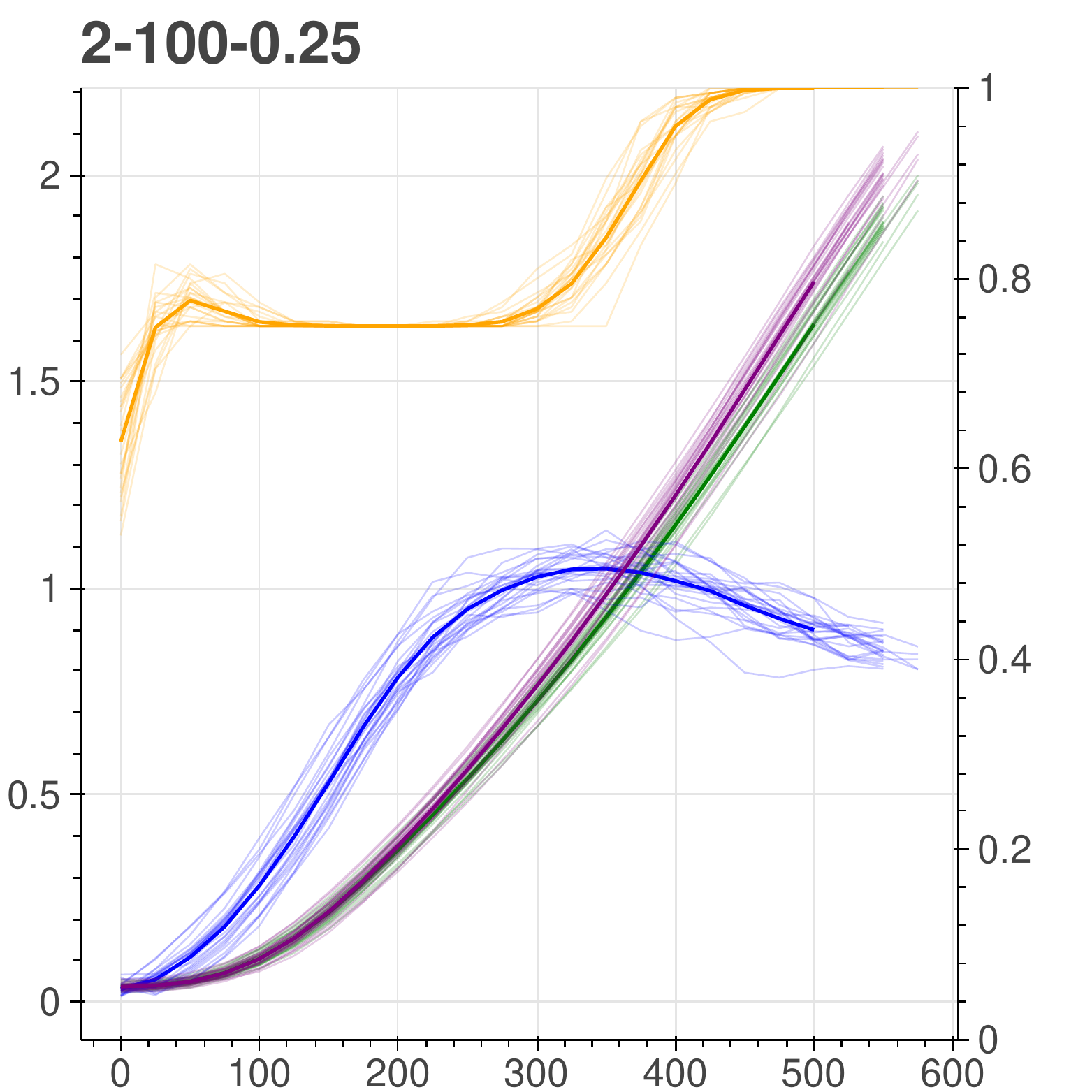}
\includegraphics[width=0.24\linewidth]{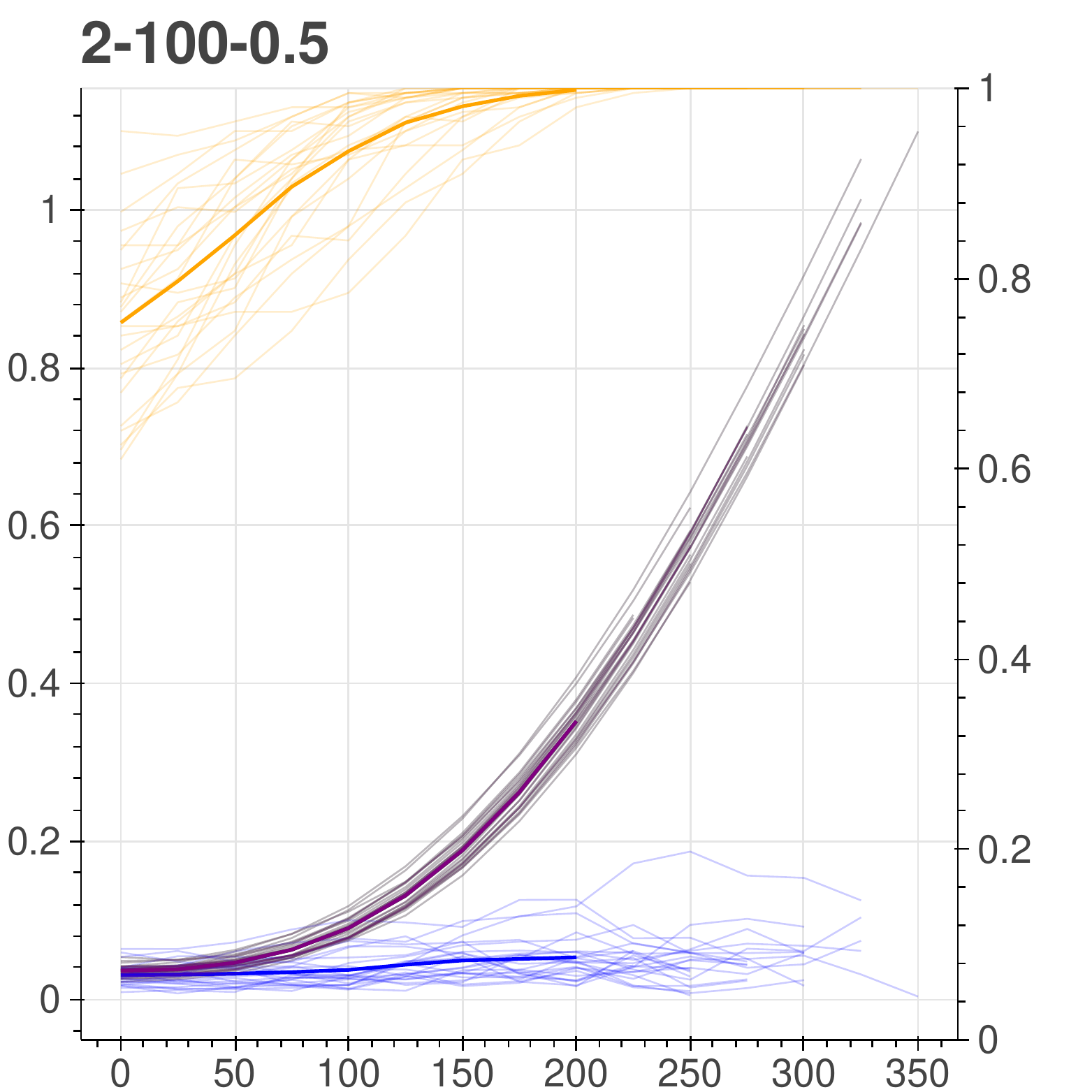}

\includegraphics[width=0.24\linewidth]{figures/wo-ptr/2-500-0.04.pdf}
\includegraphics[width=0.24\linewidth]{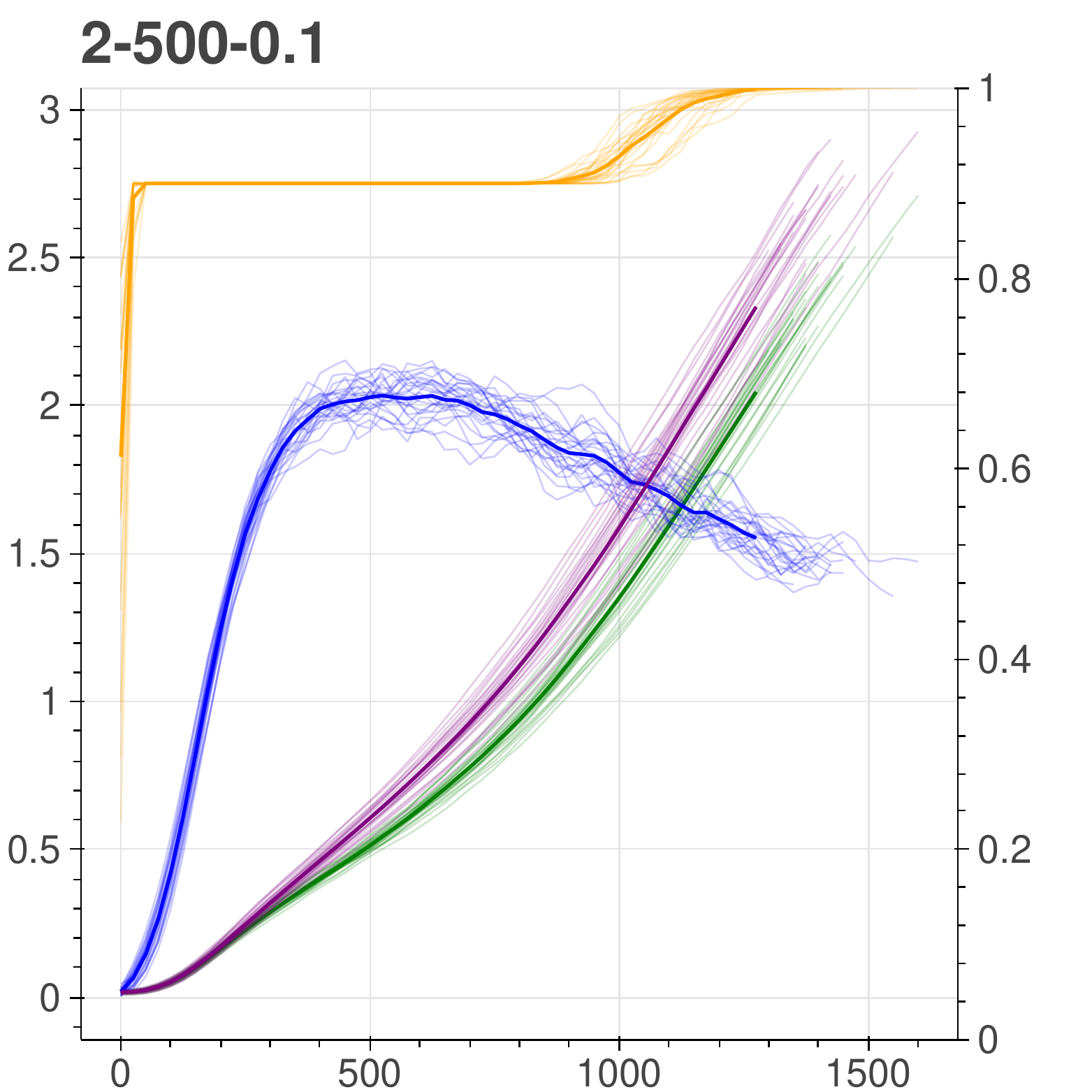}
\includegraphics[width=0.24\linewidth]{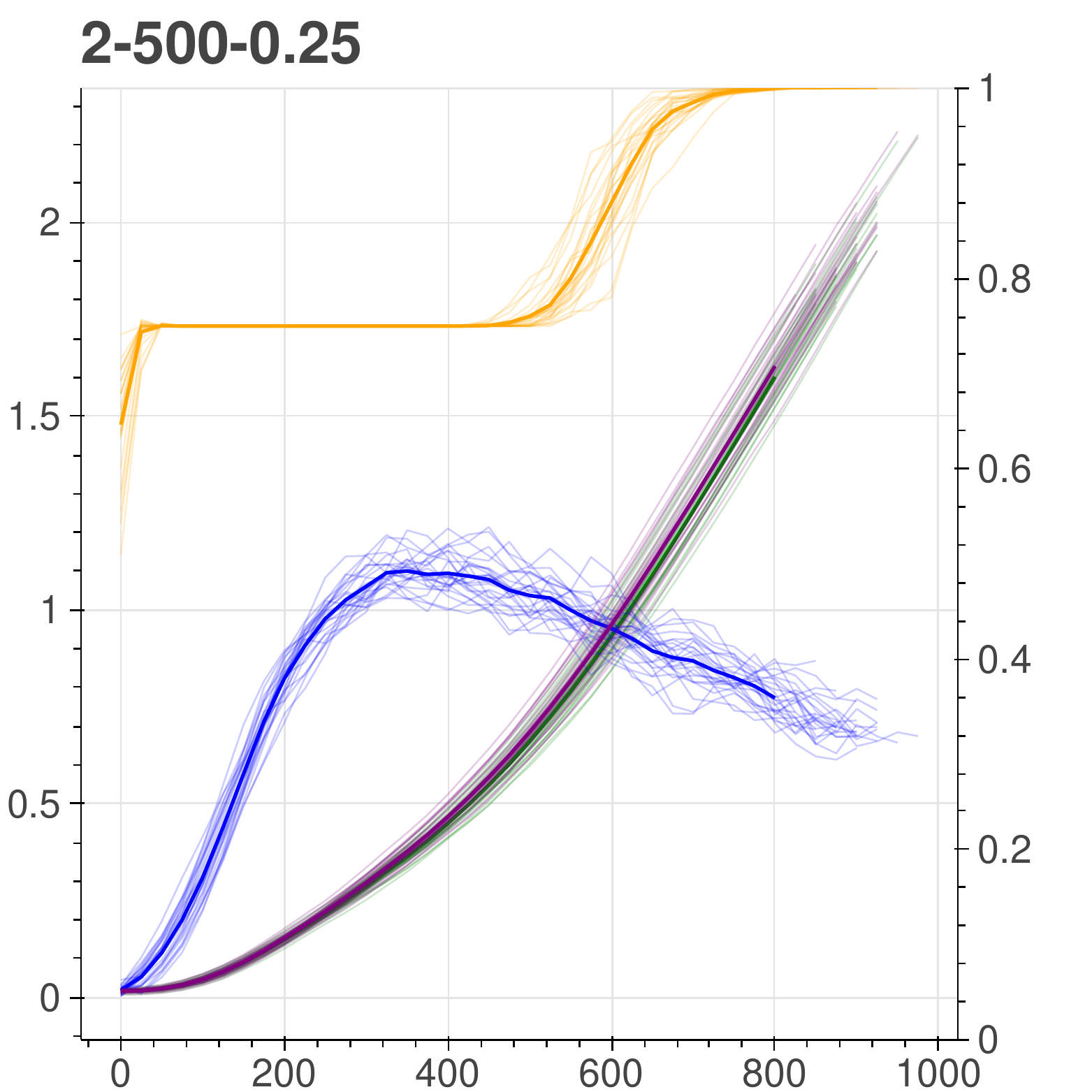}
\includegraphics[width=0.24\linewidth]{figures/wo-ptr/2-500-0.5.pdf}

\caption{The average weights over the features during training a two-layer model without pretraining.
From left to right, $\nu = 0.04, 0.10, 0.25, 0.5$.
From top to bottom, $d_1 = 50, 100, 500$.
Blue, green, purple curves represent the average weights over features in \textcolor{blue}{$X_1$}, \textcolor{DarkGreen}{$X_2$}, and \textcolor{purple}{$\dot{X}_2$} (the middle part of $X_2$) respectively. The orange curve represents the \textcolor{orange}{accuracy}.}
\label{fig:curves-wo-ptr}
\end{figure*}

\begin{figure*}
\includegraphics[width=0.24\linewidth]{figures/w-ptr/2-50-0.04.pdf}
\includegraphics[width=0.24\linewidth]{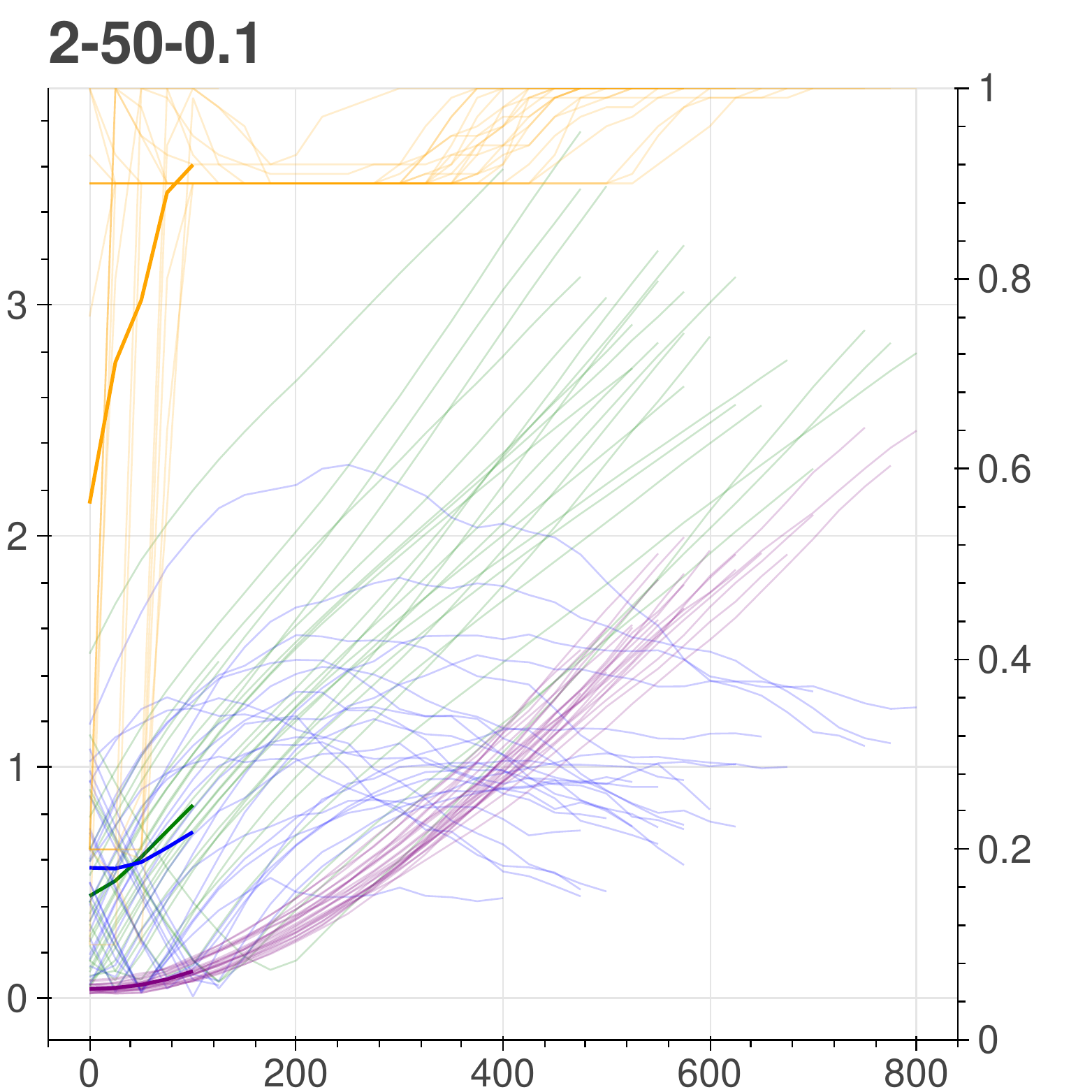}
\includegraphics[width=0.24\linewidth]{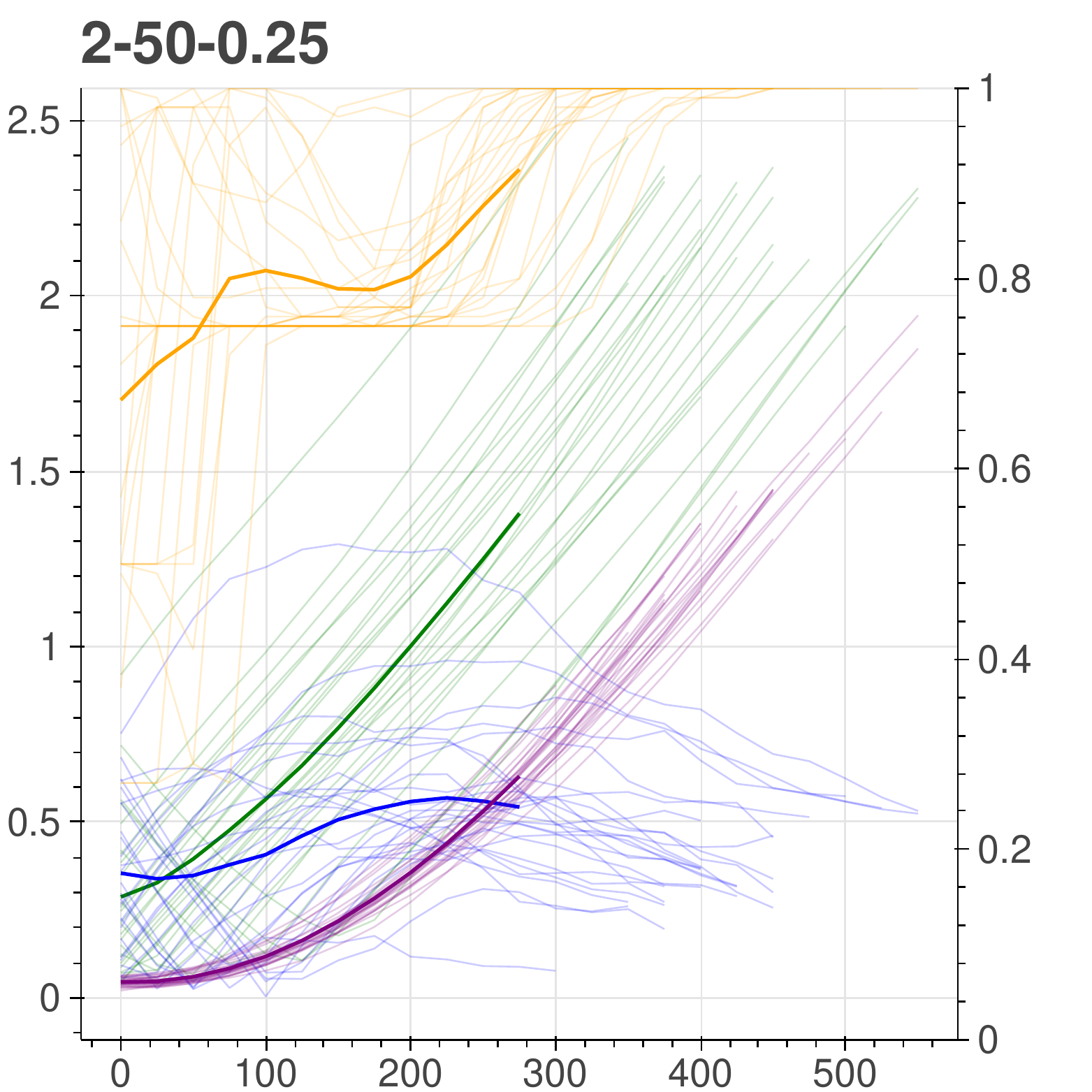}
\includegraphics[width=0.24\linewidth]{figures/w-ptr/2-50-0.5.pdf}

\includegraphics[width=0.24\linewidth]{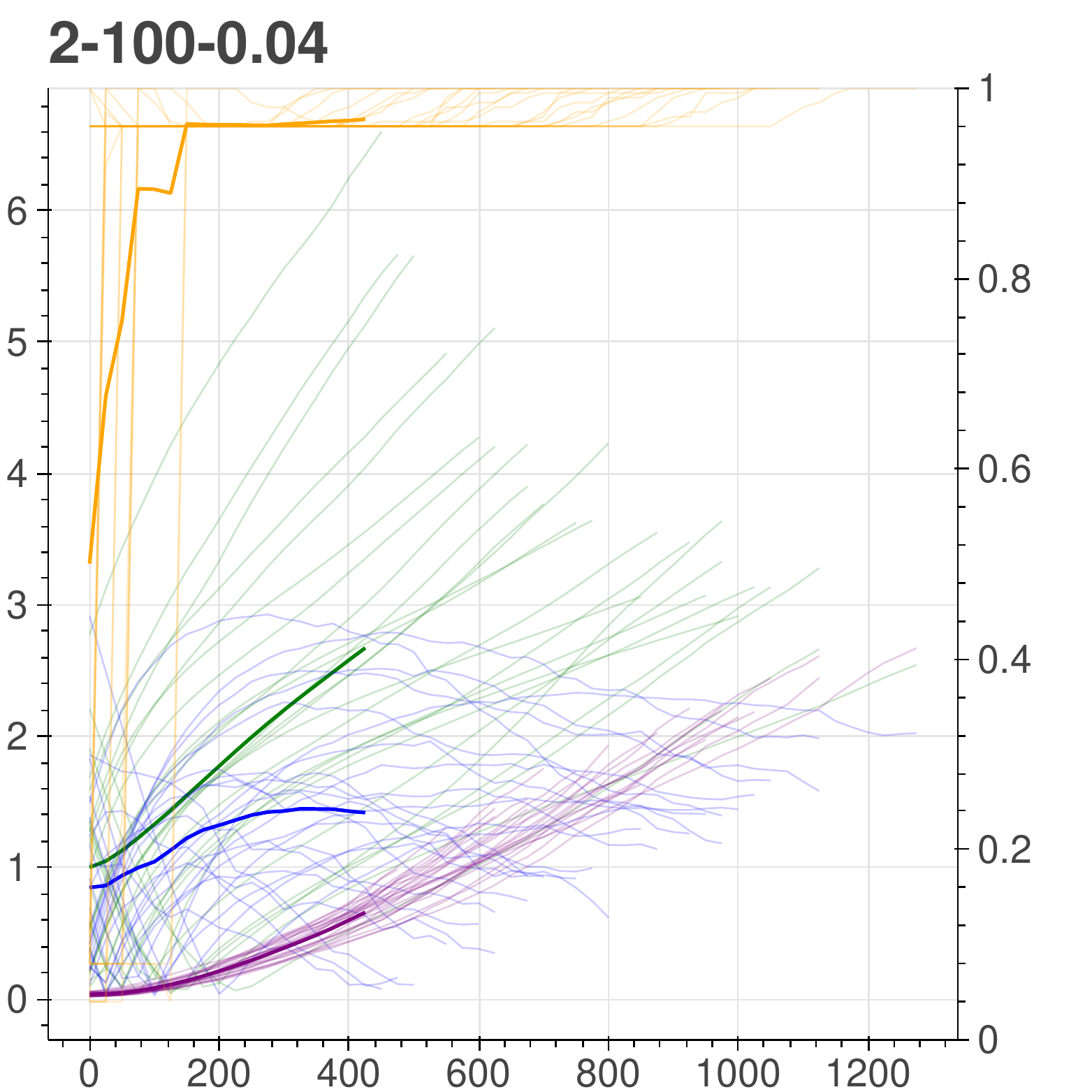}
\includegraphics[width=0.24\linewidth]{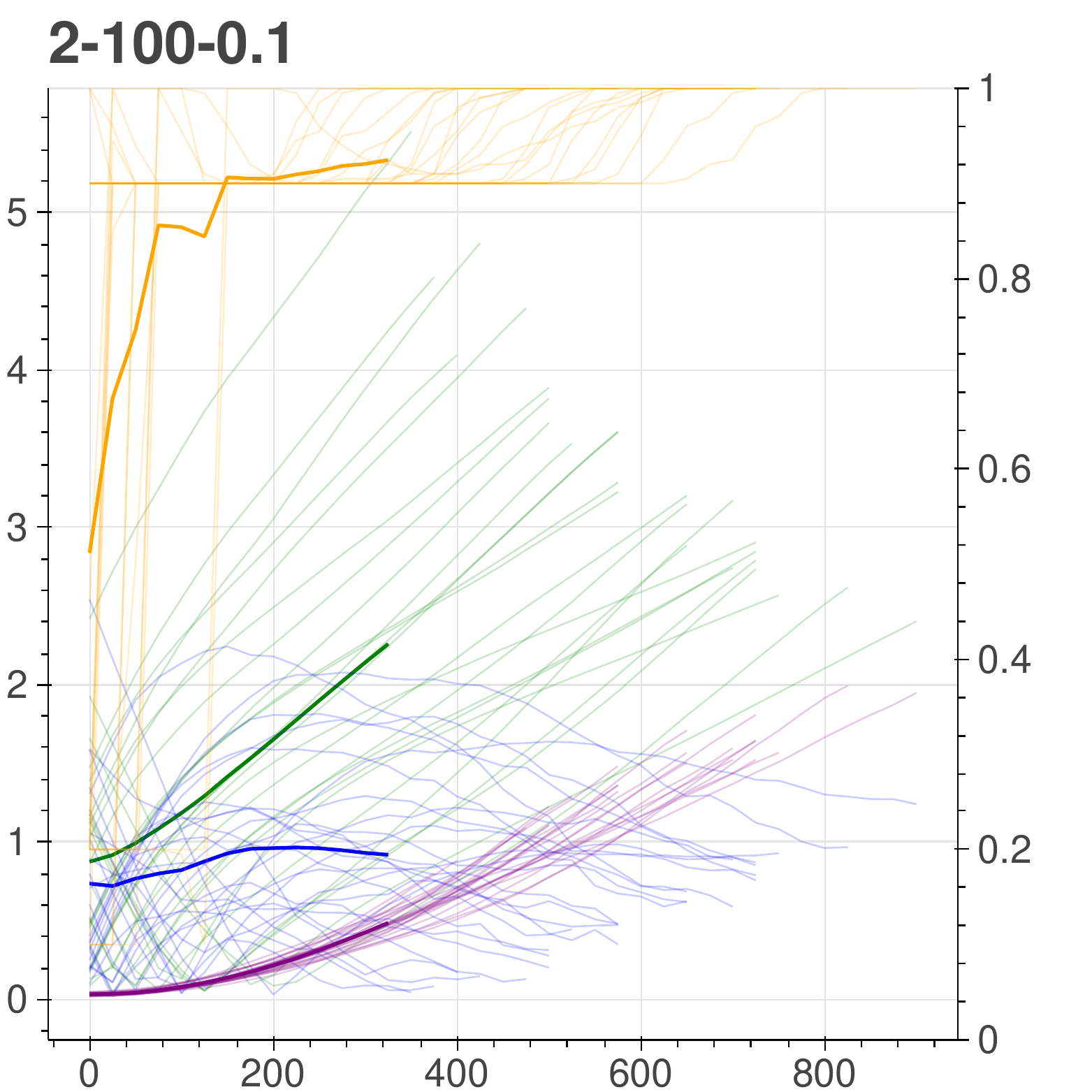}
\includegraphics[width=0.24\linewidth]{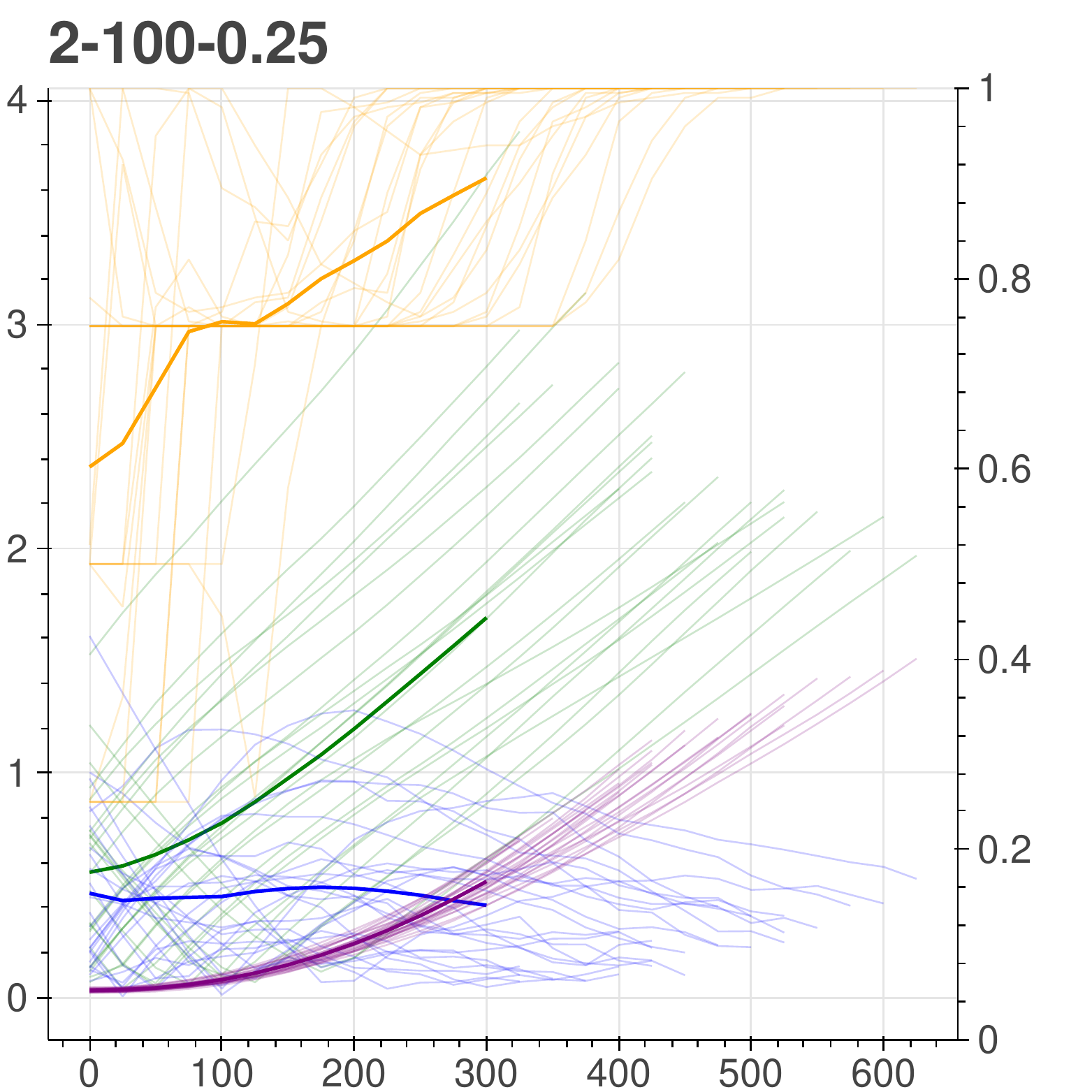}
\includegraphics[width=0.24\linewidth]{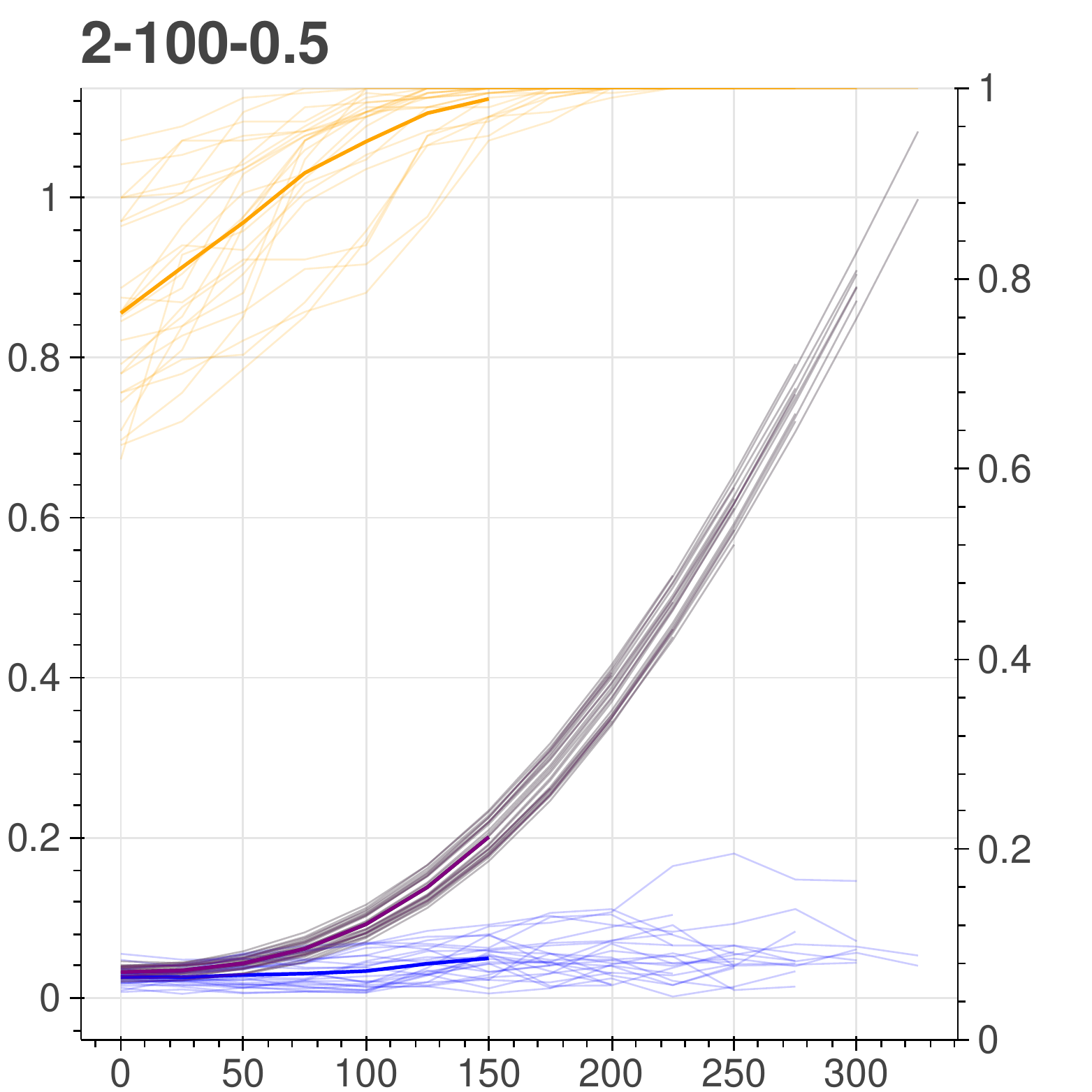}

\includegraphics[width=0.24\linewidth]{figures/w-ptr/2-500-0.04.pdf}
\includegraphics[width=0.24\linewidth]{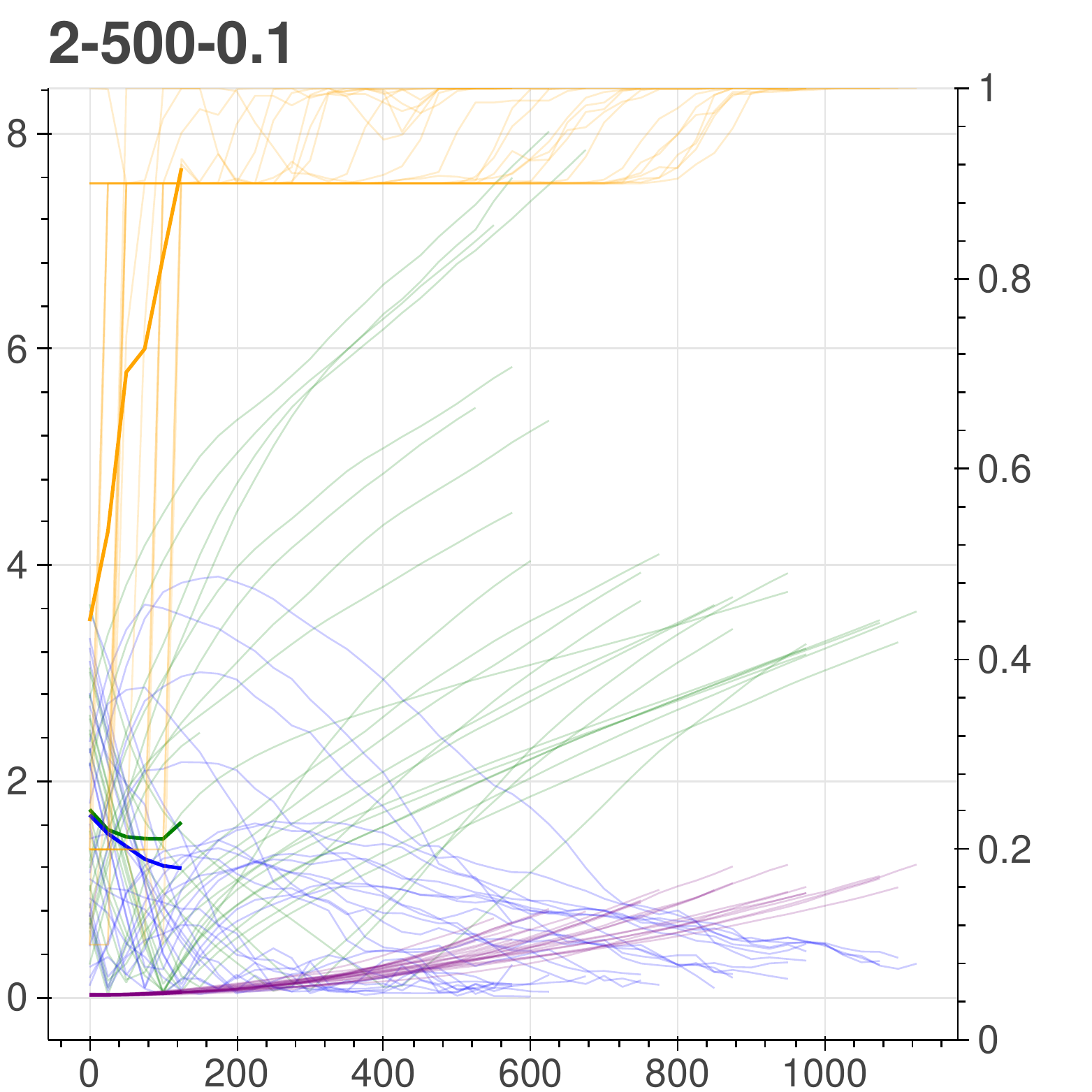}
\includegraphics[width=0.24\linewidth]{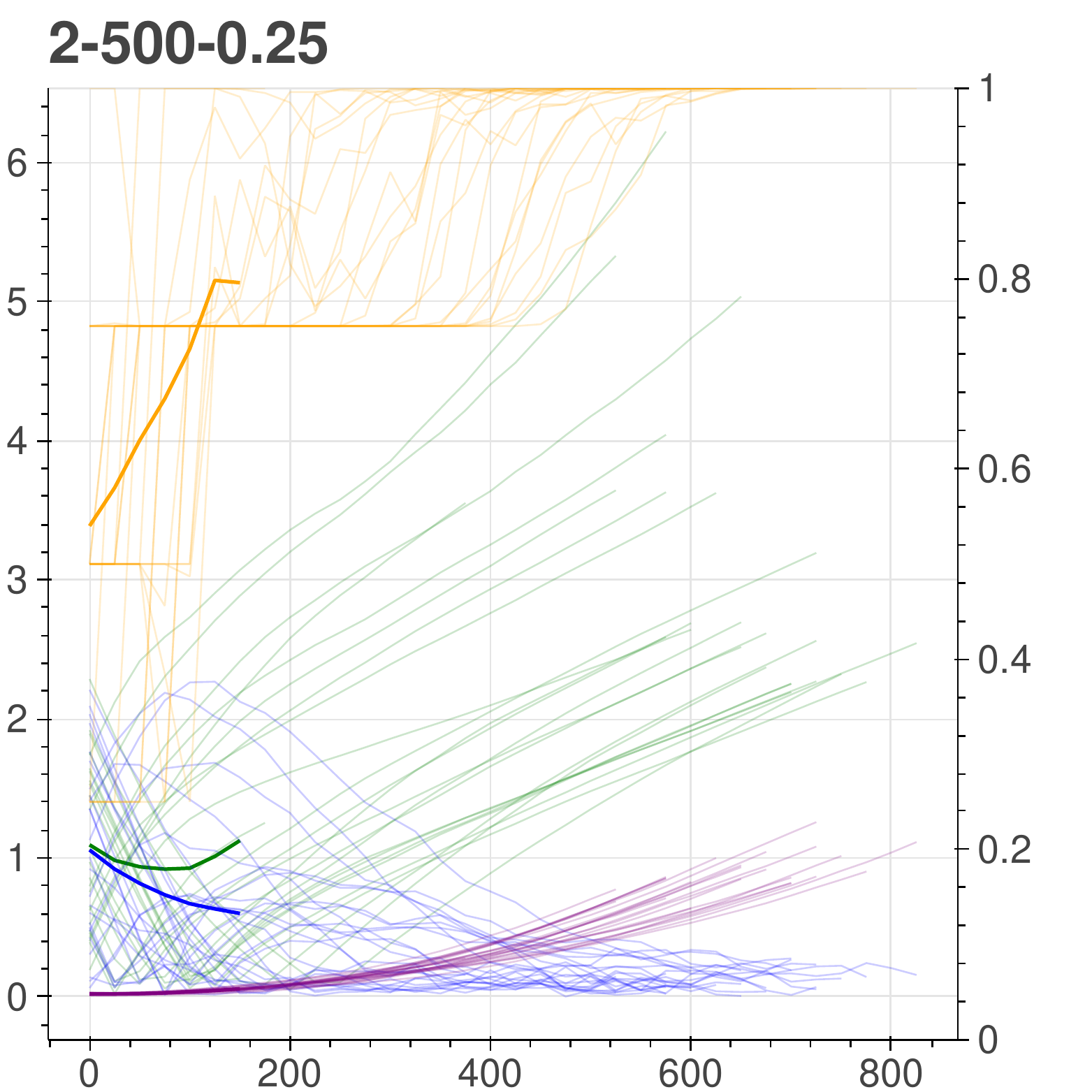}
\includegraphics[width=0.24\linewidth]{figures/w-ptr/2-500-0.5.pdf}
\caption{The average weights over the features during training a two-layer model with pretraining.
From left to right, $\nu = 0.04, 0.10, 0.25, 0.5$.
From top to bottom, $d_1 = 50, 100, 500$.
Blue, green, purple curves represent the average weights over features in \textcolor{blue}{$X_1$}, \textcolor{DarkGreen}{$X_2$}, and \textcolor{purple}{$\dot{X}_2$} (the middle part of $X_2$) respectively. The orange curve represents the \textcolor{orange}{accuracy}.}
\label{fig:curves-w-ptr}
\end{figure*}

\end{document}